\documentclass[10pt,twocolumn,twoside]{IEEEtran}
\usepackage{diagbox}
\usepackage{bm}
\usepackage{hyperref}
\usepackage{mathtools}
\usepackage{amsfonts,amssymb}
\usepackage{threeparttable}
\usepackage{makecell}
\usepackage{multirow}
\usepackage{float}
\usepackage{tcolorbox}
\usepackage{bbding}
\usepackage{booktabs}                 
\usepackage[linesnumbered,ruled,vlined]{algorithm2e}

\usepackage{graphicx}
\usepackage{amsmath}
\usepackage{textcomp}
\usepackage{xcolor}
\def\BibTeX{{\rm B\kern-.05em{\sc i\kern-.025em b}\kern-.08em
    T\kern-.1667em\lower.7ex\hbox{E}\kern-.125emX}}

\usepackage{adjustbox}
\usepackage{amssymb,mathtools}
\usepackage{subfig}
\usepackage{xurl}

\usepackage{amsthm}

\newtheorem{myDef}{Definition}
\newtheorem{myTheo}{Theorem}
\newtheorem{myProp}{Proposition}
\newtheorem{myAssum}{Assumption}

\begin{document}

\title{VFLGAN-TS: Vertical Federated Learning-based Generative Adversarial Networks for Publication of Vertically Partitioned Time-Series Data}

\author{Xun Yuan,
        Zilong Zhao,
        Prosanta Gope,~\IEEEmembership{Senior Member,~IEEE} 
        and~Biplab~Sikdar,~\IEEEmembership{Senior Member,~IEEE}
\thanks{Xun Yuan and Biplab Sikdar are with the Department of Electrical and Computer Engineering, College of Design and Engineering, National University of Singapore, Singapore. (E-mail: e0919068@u.nus.edu and bsikdar@nus.edu.sg).}
\thanks{Zilong Zhao is with the Asian Institute of Digital Finance, National University of Singapore, Singapore. (E-mail: z.zhao@nus.edu.sg).}
\thanks{Prosanta Gope is with the Department of Computer Science, University of Sheffield, United Kingdom. (E-mail: p.gope@sheffield.ac.uk).}
}

\markboth{Preprint}%
{Yuan \MakeLowercase{\textit{et al.}}: VFLGAN-TS: VFL-based GANs for Publication of Vertically Partitioned Time-Series Data}

\maketitle

\begin{abstract}
In the current artificial intelligence (AI) era, the scale and quality of the dataset play a crucial role in training a high-quality AI model. However, often original data cannot be shared due to privacy concerns and regulations. A potential solution is to release a synthetic dataset with a similar distribution to the private dataset. Nevertheless, in some scenarios, the attributes required to train an AI model are distributed among different parties, and the parties cannot share the local data for synthetic data construction due to privacy regulations. In PETS 2024, we recently introduced the \emph{first} Vertical Federated Learning-based Generative Adversarial Network (VFLGAN) for publishing vertically partitioned static data. However, VFLGAN cannot effectively handle time-series data, presenting both temporal and attribute dimensions. 
In this article, we proposed VFLGAN-TS, which combines the ideas of attribute discriminator and vertical federated learning to generate synthetic time-series data in the vertically partitioned scenario. The performance of VFLGAN-TS is close to that of its counterpart, which is trained in a centralized manner and represents the upper limit for VFLGAN-TS. To further protect privacy, we apply a Gaussian mechanism to make VFLGAN-TS satisfy an $(\epsilon,\delta)$-differential privacy. Besides, we develop an enhanced privacy auditing scheme to evaluate the potential privacy breach through the framework of VFLGAN-TS and synthetic datasets. 
\end{abstract}


\IEEEpeerreviewmaketitle

\section{Introduction}

The performance of deep-learning (DL) models is closely linked to the quality and scale of the training dataset. For example, key advancements in image perception \cite{turc2019}, language understanding \cite{turc2019}, and recommendation systems \cite{ma2019learning} are attributed to high-quality datasets like ImageNet \cite{deng2009imagenet}, expansive textual dataset \cite{conneau-etal-2017-supervised}, and Netflix rating dataset \cite{bennett2007netflix}, respectively. However, a significant challenge is the vertically partitioned scenario, where separate entities hold different attributes. For instance, a bank may have a client's financial data, while health records are maintained by hospitals or insurers. Integrating these diverse attributes can offer a more comprehensive understanding of customers, thereby enhancing decision-making, as highlighted in several studies \cite{8667703, Xue01}.

Although integrating dispersed data attributes offers clear benefits, practical implementation faces obstacles due to privacy concerns and stringent data protection regulations like GDPR \cite{voigt2017eu} that prohibit sharing the original data. A potential solution to this challenge is to publish a synthetic dataset that reflects the private dataset's distribution without revealing any private information. Generating synthetic datasets usually requires access to all the attributes of the private datasets. However, in the vertically partitioned scenario, each party holds part of the attributes and cannot share its local attributes. This problem is the publication of vertically partitioned data. The publication of vertically partitioned data aims to release a synthetic dataset that mimics the distribution of the private dataset, whose attributes are held by different parties and cannot be shared among them.

In our previous study \cite{Xun_vfl}, we introduced the first Vertical Federated Learning-based Generative Adversarial Networks (VFLGAN) to publish vertically partitioned static data. However, VFLGAN struggles with generating time-series data due to its inability to simultaneously learn the correlation along the temporal and attribute dimensions, as shown in the experimental results in Section \ref{sec:evaluation results}. To address this issue, we developed the first generative model for vertically partitioned time-series data, Vertical Federated Learning-based Generative Adversarial Networks for Time Series (VFLGAN-TS). VFLGAN-TS integrates the ideas of the attribute discriminator proposed in \cite{seyfi2022generating} and Vertical Federated Learning (VFL) \cite{liu2022vertical}. The attribute discriminator effectively learns the temporal correlations, and VFL satisfies the privacy constraints of vertically partitioned scenarios and learns the attribute correlations.

As discussed in \cite{Theresa01}, synthetic datasets are not safe from privacy attacks. To counteract such vulnerabilities, Differential
Privacy (DP) \cite{dwork2006} offers a promising privacy protection strategy. In this paper, we employ the Gaussian mechanism proposed in \cite{Xun_vfl} to make VFLGAN-TS satisfy an $(\epsilon, \delta)$-DP to further protect privacy, and we name the differentially private version with DPVFLGAN-TS. On the other hand, although DP can offer worst-case privacy assurances \cite{mehner2021towards}, most real-world datasets do not contain the worst-case sample. Thus, in \cite{Xun_vfl}, we proposed a practical auditing scheme to evaluate the potential privacy breach through synthetic datasets, but our test shows that the scheme is not effective for time-series synthetic data (details shown in Table \ref{tab:ASSD}). In this paper, we enhance the auditing scheme to evaluate the potential for privacy breach through the VFLGAN-TS framework and the generated synthetic time-series datasets. To summarize, our contributions in this article are as follows:
\label{sec:contributions}


\begin{itemize}
    
    \item[$\bullet$] We introduce VFLGAN-TS, the \emph{first} solution for the publication of vertically partitioned time-series data. Besides, we adapt the Gaussian mechanism proposed in  \cite{Xun_vfl} to equip VFLGAN-TS with $(\epsilon,\delta)$-DP.
    
    \item[$\bullet$] We enhance the auditing scheme proposed in \cite{Xun_vfl} to evaluate potential privacy breaches through VFLGAN-TS framework and synthetic datasets.
  
    \item[$\bullet$] Extensive experiments were conducted to evaluate the quality and privacy breaches of synthetic datasets generated by VFLGAN-TS.

    \item[$\bullet$] The datasets and implementation details have been released through Github$\footnote{\url{https://github.com/YuanXun2024/VFLGAN-TS}}$.
    
\end{itemize}

\section{Related Work}

In this section, we introduce related work about the publication of vertically partitioned data, Generative Adversarial Networks (GANs) for time-series data, differentially private mechanisms for GANs, and privacy auditing methods.

\subsection{Publication of Vertically Partitioned Data}

DistDiffGen \cite{6517175} is a secure two-party algorithm using the exponential mechanism to achieve $\epsilon$-DP. However, it is only suitable for classification tasks and lacks utility for other common data analysis tasks \cite{8667703}. The authors of \cite{8667703} proposed DPLT, which also satisfies $\epsilon$-DP but is limited to discrete datasets and evenly distributes the privacy budget across all attributes. As \cite{Xue01} notes, increased data dimensionality can exponentially raise the noise scale, leading to significant utility loss. The first GAN-based approach, VertiGAN \cite{Xue01}, utilizes FedAvg \cite{mcmahan2017communication} to meet privacy needs in vertically partitioned scenarios but struggles with learning attribute correlation. VFLGAN, integrating VFL and WGAN\_GP \cite{gulrajani2017improved}, was proposed in \cite{Xun_vfl} as the most effective method for vertically partitioned static data and the differentially private version satisfies $(\epsilon,\delta)$-DP. However, none of these methods suits time-series data.

\subsection{GANs for Time-Series Data}

In \cite{mogren2016c,esteban2017real,ramponi2018t}, the GAN framework has been directly applied to time-series data, but the effectiveness in capturing temporal correlations is limited \cite{yoon2019time}. To address this, TimeGAN \cite{yoon2019time} integrates supervised learning within the GAN framework to better capture temporal dynamics. Fourier Flows \cite{alaa2021generative} employ a discrete Fourier transform (DFT) to convert time series into fixed-length spectral representations, followed by a chain of spectral filters leading to an exact likelihood optimization. GT-GAN \cite{jeon2022gt} integrates various techniques, including GANs, neural ordinary/controlled differential equations, and continuous time-flow processes, into a single framework for time series synthesis. Still, its complexity makes adaptation to the VFL framework challenging. CosciGAN \cite{seyfi2022generating} advances in generating time series data by using channel discriminators to learn temporal distributions of each attribute and a central discriminator to understand attribute intercorrelations. CosciGAN outperforms both TimeGAN and Fourier Flows as shown in \cite{seyfi2022generating}. In this paper, we combine the channel/attribute discriminator and VFL framework to publish vertically partitioned time-series data.

\subsection{Differentially Private Mechanisms for GANs}

In \cite{xie2018differentially, zhang2018differentially}, the authors proposed variants of DPSGD \cite{shokri2015privacy} to train discriminators privately, while using non-private SGD to train generators. PATE-GAN \cite{jordon2018pate} utilizes the PATE framework \cite{papernot2018scalable} to ensure differential privacy. It begins by training multiple non-private discriminators with non-overlapped subsets. These discriminators are then used to train a student discriminator that satisfies $(\epsilon, \delta)$-DP, which in turn trains the generator. Another approach, GS-WGAN \cite{chen2020gs}, trains discriminators non-privately with non-overlapped subsets but sanitizes the backward gradients between discriminators and the generator using a Gaussian mechanism to achieve $(\epsilon,\delta)$-DP. However, as discussed in \cite{Xun_vfl}, the above methods are unsuitable for vertically partitioned scenarios. This paper adapts the Gaussian mechanism proposed in \cite{Xun_vfl} to VFLGAN-TS, ensuring it has an $(\epsilon, \delta)$-DP.

\subsection{Privacy Auditing Methods}

Privacy auditing encompasses two primary research directions. The first involves estimating the lower bounds of $\epsilon$ for $\epsilon$-DP using poisoning samples, as explored in \cite{jagielski2020auditing, lu2022general, steinke2024privacy, krishna2023unleashing}. However, these methods are designed for classification tasks and do not align with our needs, and they also fail to capture the privacy risks associated with real training samples. The second line of research focuses on launching Membership Inference (MI) attacks on synthetic datasets, as seen in \cite{hayes2017logan, hilprecht2019monte, chen2020gan, van2023membership}, using the success rate of these attacks to estimate privacy breaches. Our previous research \cite{Xun_vfl} introduces a more effective privacy auditing scheme, ASSD, which combines the shadow-model attack \cite{Theresa01} and the leave-one-out setting \cite{Jiayuan01}, and significantly surpasses the performance of existing methods. In this article, we enhance ASSD and use it for privacy analysis.

\section{Prelimineries}
This section provides preliminaries of GANs, the auditing scheme of VFLGAN \cite{Xun_vfl}, and differential privacy to facilitate a comprehensive understanding of the proposed VFLGAN-TS, DPVFLGAN-TS, and the enhanced auditing scheme.

\subsection{GANs for Time-Series Data}

This section begins with GANs for static data. Given a dataset $X \in \mathbb{R}^{N \times |A|}$ where $N$ is the sample number and $A$ is the attribute set, the generator of GAN, $G$, aims to generate synthetic data $\tilde{\boldsymbol{x}}=G(\boldsymbol{z})\in \mathbb{R}^{|A|}$ and the distribution of the synthetic data $P_{\tilde{\boldsymbol{x}}}$ should be similar the distribution of the real data, $P_{\boldsymbol{x}}$ where $\boldsymbol{x}\in X$. 
The input $\boldsymbol{z}$ is a vector sampled from a simple distribution, such as the uniform or Gaussian distribution. The above objectives can be achieved with a discriminator, $D$. The generator and discriminator are trained through a competing game, where the discriminator is trained to distinguish $\boldsymbol{x}$ and $\tilde{\boldsymbol{x}}$. In contrast, the generator is trained to generate high-quality $\tilde{\boldsymbol{x}}$ to fool the discriminator. 
The game between the generator and the discriminator can be formally expressed as the following min-max objective \cite{goodfellow2014generative},
\begin{equation}
\label{eq:min-max objective}
\min _G \max _D \underset{\boldsymbol{x} \sim P}{\mathbb{E}}[\log (D(\boldsymbol{x}))]+\underset{\tilde{\boldsymbol{x}} \sim P_{G(\boldsymbol{z})}}{\mathbb{E}}[\log (1-D(\tilde{\boldsymbol{x}}))],
\end{equation}
where $D(\cdot)$ denotes the calculation of the discriminator $D$.

While static data only has the attribute dimension, time-series data, $X \in \mathbb{R}^{N \times |A|\times T}$, has both the attribute and temporal dimensions, where $T$ denotes the time steps of each attribute. Normal GANs \cite{mogren2016c,esteban2017real,ramponi2018t} with training objective \eqref{eq:min-max objective} exhibit inadequate performance for generating time-series data \cite{yoon2019time,seyfi2022generating}.
This inadequacy stems from their inability to differentiate effectively between correlations among temporal and feature dimensions. In \cite{seyfi2022generating}, the authors proposed CosciGAN to solve this problem. In CosciGAN, there are $|A|$ attribute generators where $A$ is the attribute set, $|A|$ attribute discriminators, and one central discriminator. Each attribute generator (e.g., $G_i, i\in\{1,\cdots,|A|\}$) is responsible for generating one synthetic attribute ($\tilde{\boldsymbol{x}}_i\in \mathbb{R}^{N \times T}$), and each attribute discriminator (e.g., $D_i, i\in\{1,\cdots,|A|\}$) is responsible for optimizing $G_i$ to make the distribution of the synthetic attribute similar to that of real attribute, i.e., $P_{\tilde{\boldsymbol{x}}_i}\approx P_{{\boldsymbol{x}}_i}$. On the other hand, the central discriminator, $D_C$, is responsible for optimizing all generators to make the distribution of the synthetic data, $\tilde{\boldsymbol{x}}=[\tilde{\boldsymbol{x}}_1,\cdots,\tilde{\boldsymbol{x}}_{|A|}]$, similar to that of real data $\boldsymbol{x}$, i.e., $P_{\tilde{\boldsymbol{x}}}\approx P_{\boldsymbol{x}}$. Following the theoretical analysis in \cite{goodfellow2014generative}, the training objectives of CosciGAN can be expressed as,
\begin{equation}
\begin{aligned}
\label{eq:objective_CosciGAN}
\min _G \max _D \sum_{i=1}^{|A|}\left(\mathbb{E}[\log (D_i(\boldsymbol{x}_i))]+\mathbb{E}[\log (1-D_i(\tilde{\boldsymbol{x}}_i))])\right) \\ +  \lambda\left(\mathbb{E}[\log (D_C(\boldsymbol{x}))]+\mathbb{E}[\log (1-D_C(\tilde{\boldsymbol{x}}))])\right),
\end{aligned}
\end{equation}
where $G=\{G_1,\cdots, G_{|A|}\}$, $D=\{D_1,\cdots,D_{|A|},D_C\}$, and $\lambda$ is a balancing coefficient.

\subsection{Auditing Scheme for Synthetic Datasets} \label{sec:pre_ASSD}
In our previous study \cite{Xun_vfl}, we proposed a shadow model-based membership inference (MI) attack with the leave-one-out assumption \cite{Jiayuan01}, named ASSD, to audit privacy breaches through synthetic datasets, which is described in Algorithm \ref{Game:MI}. First, the challenger trains a generator $G_0$ with the whole dataset excluding target sample $x_t$ and trains a generator $G_1$ with the whole dataset. Then, the challenger generates synthetic datasets $\tilde{x}_0$ and $\tilde{X}_1$ with $G_0$ and $G_1$, respectively. Next, the challenger flips a random and unbiased coin $b\in \{0,1\}$ and sends the synthetic dataset $\tilde{X}_b$ and target record $\boldsymbol{x}_t$ to the adversary that outputs $\hat{b}$. Last, if $\hat{b}==b$, the adversary wins. Otherwise, the adversary loses. 

Algorithm \ref{alg:MI adversary} describes the training process of the adversary in Algorithm \ref{Game:MI}. First, the adversary trains $M$ generators $G_{0_{1:M}}$ with the whole dataset excluding target sample $x_t$ and trains $M$ generators $G_{1_{1:M}}$ with the whole dataset. Then, the adversary generates $M$ synthetic datasets $\tilde{X}_{0_{1:M}}$ and $M$ synthetic datasets $\tilde{X}_{1_{1:M}}$ with $G_{0_{1:M}}$ and $G_{1_{1:M}}$, respectively. Next, the adversary extracts features from $\tilde{X}_{0_{1:M}}$ and $\tilde{X}_{1_{1:M}}$. Last, the adversary trains the Adversary model (random forest) with the features to distinguish whether $x_t$ is included in the training dataset.

\begin{algorithm}[htbp]
\footnotesize
\caption{Membership Inference Attack}\label{Game:MI}
\SetKwInOut{Input}{Input}
\SetKwInOut{Output}{Output}
\Input {Training algorithm $\mathcal{T}$; dataset $X$; target record $\boldsymbol{x}_t$; unbiased coin $b$; fresh random seeds $s_{0}$ and $s_{1}$; generators of VFLGAN $G_i$; synthetic dataset $\tilde X$; Gaussian noise $\boldsymbol{z}$.}
\Output {Success or failure.}
$G_0 \stackrel{s_{0}}{\longleftarrow} \mathcal{T}(X \setminus \boldsymbol{x}_t)$ \&
$G_1 \stackrel{s_{1}}{\longleftarrow} \mathcal{T}(X)$ \\
$\tilde{X}_0 \longleftarrow G_0(\boldsymbol{z})$ \& $\tilde{X}_1 \longleftarrow G_1(\boldsymbol{z})$ \\
$\hat{b} \longleftarrow \mathcal{A}(\tilde{X}_b, \boldsymbol{x}_t)$ where $b \sim \{0,1\}$ \\

\uIf{$\hat{b}==b$}{
    Output success \;
  }
  \Else{
    Output failure \;
  }
\end{algorithm}
\vspace{-7mm}

\begin{algorithm}[htbp]
\footnotesize
\caption{Training the Adversary of MI Attack} \label{alg:MI adversary}
\SetKwInOut{Input}{Input}
\SetKwInOut{Output}{Output}
\Input {Training algorithms $\mathcal{T}$ and $\mathcal{T}_\mathcal{A}$; dataset $X$; target record $\boldsymbol{x}_t$; random seeds $s_{0_{1:M}}$ and $s_{1_{1:M}}$; synthetic dataset $\tilde X$; feature extraction function $Extr(\cdot)$.}
\Output {Trained $\mathcal{A}$.}
$G_{0_{1:M}} \stackrel{s_{0_{1:M}}}{\longleftarrow} \mathcal{T}(X \setminus \boldsymbol{x}_t)$ \& $G_{1_{1:M}} \stackrel{s_{1_{1:M}}}{\longleftarrow} \mathcal{T}(X)$\\
$\tilde{X}_{0_{1:M}} \longleftarrow G_{0_{1:M}}$ \& $\tilde{X}_{1_{1:M}} \longleftarrow G_{1_{1:M}}$\\
$Feat_{0_{1:M}} \longleftarrow Extr(\tilde{X}_{0_{1:M}})$ \& $Feat_{1_{1:M}} \longleftarrow Extr(\tilde{X}_{1_{1:M}})$\\
$\mathcal{A} \longleftarrow \mathcal{T}_\mathcal{A}(Feat_{0_{1:M}}, Feat_{1_{1:M}})$\\
\end{algorithm}
\vspace{-6mm}

\subsection{Differential Privacy}
Differential privacy (DP) \cite{dwork2006} provides a rigorous privacy guarantee that can be quantitatively analyzed. The $\epsilon$-DP is defined as follows.
\begin{myDef}
\label{Def1}
    ($\epsilon$-DP). A randomized mechanism $f: D \rightarrow R$ satisfies $\epsilon$-differential privacy ($\epsilon$-DP) if for any adjacent $D, D^\prime \in \mathcal{D}$ and $S \subset R$ 
    \centerline{$Pr[f(D)\in S] \leq e^\epsilon Pr[f(D^\prime)\in S]$.}
\end{myDef}
The most commonly used DP in the literature is a relaxed version of the $\epsilon$-DP, which allows the mechanism not to satisfy $\epsilon$-DP with a small probability, $\delta$. The relaxed version, $(\epsilon,\delta)$-DP \cite{dwork2006our}, is defined as follows.
\begin{myDef}
\label{Def2}
    (($\epsilon,\delta$)-DP). A randomized mechanism $f: D \rightarrow R$ provides ($\epsilon,\delta$)-differential privacy (($\epsilon,\delta$)-DP) if for any adjacent $D, D^\prime \in \mathcal{D}$ and $S \subset R$ 
    
    \centerline{$Pr[f(D)\in S] \leq e^\epsilon Pr[f(D^\prime)\in S]+\delta$.}
\end{myDef}

In \cite{mironov2017renyi}, the $\alpha$-R\'enyi divergences between $f(D)$ and $f(D^\prime)$ are applied to define R\'enyi Differential Privacy (RDP) which is a generalization of differential privacy. ($\alpha,\epsilon(\alpha)$)-RDP is defined as follows.
\begin{myDef}
\label{Def3}
    ($(\alpha,\epsilon(\alpha))$-RDP). A randomized mechanism $f:D \rightarrow R$ is said to have $\epsilon(\alpha)$-R\'enyi differential privacy of order $\alpha$, or $(\alpha,\epsilon(\alpha))$-RDP for short if for any adjacent $D, D^\prime \in \mathcal{D}$ it holds that
    
    \centerline{$
D_\alpha\left(f(D) \| f\left(D^{\prime}\right)\right)=\frac{1}{\alpha-1} \log \mathbb{E}_{x \sim f(D)}\left[\left(\frac{\operatorname{Pr}[f(D)=x]}{\operatorname{Pr}\left[f\left(D^{\prime}\right)=x\right]}\right)^{\alpha-1}\right] \leq \epsilon
$.}
\end{myDef}

    
    

The (R)DP budget should be accumulated if we apply multiple mechanisms to process the data sequentially as we train deep learning (DL) models for multiple iterations. We can calculate the accumulated RDP budget by the following proposition \cite{mironov2017renyi}.

\begin{myProp}
\label{Prop: RDP composition}
    (Composition of RDP) Let $f : D \rightarrow R_1$ be $(\alpha, \epsilon_1)$-RDP and $g : R_1 \times D \rightarrow R_2$ be $(\alpha,\epsilon_2)$-RDP. Then the mechanism defined as $(X, Y)$, where $X \sim f(D)$ and $Y \sim g(X, D)$, satisfies $(\alpha, \epsilon_1 + \epsilon_2)$-RDP.
\end{myProp}
According to the following proposition, RDP can be converted to $(\epsilon,\delta)$-DP and the proof can be found in \cite{mironov2017renyi}.
\begin{myProp}
\label{Prop: RDP to DP}
    (From RDP to $(\epsilon,\delta)$-DP) If $f$ is an $(\alpha, \epsilon(\alpha))$-RDP mechanism, it also satisfies $(\epsilon(\alpha) + \frac{log1/\delta}{\alpha-1}, \delta)$-DP for any $0 < \delta < 1$.
\end{myProp}
According to Proposition \ref{Prop: RDP to DP}, given a $\delta$ we can get a tight $(\epsilon^\prime, \delta)$-DP bound by
\begin{equation}
    \label{eq:RDP to DP}
    \epsilon^\prime = \min_\alpha(\epsilon(\alpha) + \frac{log1/\delta}{\alpha-1}).
\end{equation}
A tighter upper bound on RDP is provided in \cite{wang2019subsampled} by considering the combination of the subsampling procedure and random mechanism. This is important for differentially private DL since DL models are mostly updated according to a subsampled mini-batch of data. The enhanced RDP bound can be calculated according to the following proposition, and the proof can be found in \cite{wang2019subsampled}.

\begin{myProp}
\label{theorem: subsample RDP}
(RDP for Subsampled Mechanisms). Given a dataset of $n$ points drawn from a domain $\mathcal{X}$ and a (randomized) mechanism $\mathcal{M}$ that takes an input from $\mathcal{X}^m$ for $m \leq n$, let the randomized algorithm  $\mathcal{M} \circ$ \text{subsample} be defined as (1) subsample: subsample without replacement $m$ datapoints of the dataset (sampling rate $\gamma$ = $m/n$), and (2) apply $\mathcal{M}$: a randomized algorithm taking the subsampled dataset as the input. For all integers $\alpha$ $\geq$ 2, if $\mathcal{M}$ obeys ($\alpha, \epsilon(\alpha)$)-RDP, then this new randomized algorithm $\mathcal{M} \circ$ \text{subsample} obeys ($\alpha, \epsilon^\prime(\alpha)$)-RDP where,

\centerline{$
\begin{aligned}
\epsilon^{\prime}(\alpha) \leq & \frac{1}{\alpha-1} \log \Bigg(1+\gamma^2\bigg(\begin{array}{c}
\alpha \\ 2 \end{array}\bigg) \min \\
&\bigg\{4\bigg(e^{\epsilon(2)}-1\bigg),  e^{\epsilon(2)} \min     
\bigg\{2,\left(e^{\epsilon(\infty)}-1\right)^2\bigg\}\bigg\} \\
+ & \sum_{j=3}^\alpha \gamma^j\left(\begin{array}{c}
\alpha \\
j
\end{array}\right) e^{(j-1) \epsilon(j)} \min \bigg\{2,\bigg(e^{\epsilon(\infty)}-1\bigg)^j\bigg\}\Bigg)
\end{aligned}
$}
\end{myProp}


\section{Proposed VFLGAN-TS}
This section begins by defining the publication of vertically partitioned time-series data. Next, we introduce the framework of VFLGAN-TS and separate the problem into two distinct objectives to learn the correlation along temporal and attribute dimensions. Then, we describe the training process of VFLGAN-TS and DPVFLGAN-TS. Last, we introduce the enhanced auditing scheme.

\subsection{Problem Formulation}

In this paper, we consider a scenario including $M$ non-colluding parties. Each party, $P_i$, maintains a private time-series dataset, $X_i\in \mathbb{R}^{N \times |A_i| \times T}$ where $N$ denotes the sample number, $A_i$ denotes the attribute set of $X_i$, and $T$ denotes the time-series length. We have the following two assumptions for this scenario.
\begin{myAssum}
Samples in each private dataset with the same index correspond to the same object. This alignment can be achieved by using private set intersection protocols \cite{huang2012private, chen2017fast}.
\end{myAssum}
\begin{myAssum}
There is no common attribute in different parties. 
\end{myAssum}

We can obtain a new dataset with the above assumptions by combining the $M$ private datasets, i.e., $X=[X_1, \cdots, X_M]$. The objectives of vertically partitioned data publication are to release a synthetic dataset $\tilde{X}$ without sharing the local private data and the sample ($\tilde{\boldsymbol{x}}\in \tilde{X}$) in the synthetic dataset should follow a similar distribution to the sample ($\boldsymbol{x} \in X$) in the private dataset. The problem can be formulated as,
\begin{equation}
    \label{eq:problem objective}
     \min_\theta \mathcal{D}( P_{{\tilde{\boldsymbol{x}}}},P_{\boldsymbol{x}}),
\end{equation}
where $\theta$ denotes the parameters to generate $\tilde{\boldsymbol{x}}$ and $\mathcal{D}$ is any suitable measure of the distance between two distributions. 

\subsection{Framework of VFLGAN-TS}

\begin{figure}
\center{\includegraphics[width=\linewidth, scale=1.]{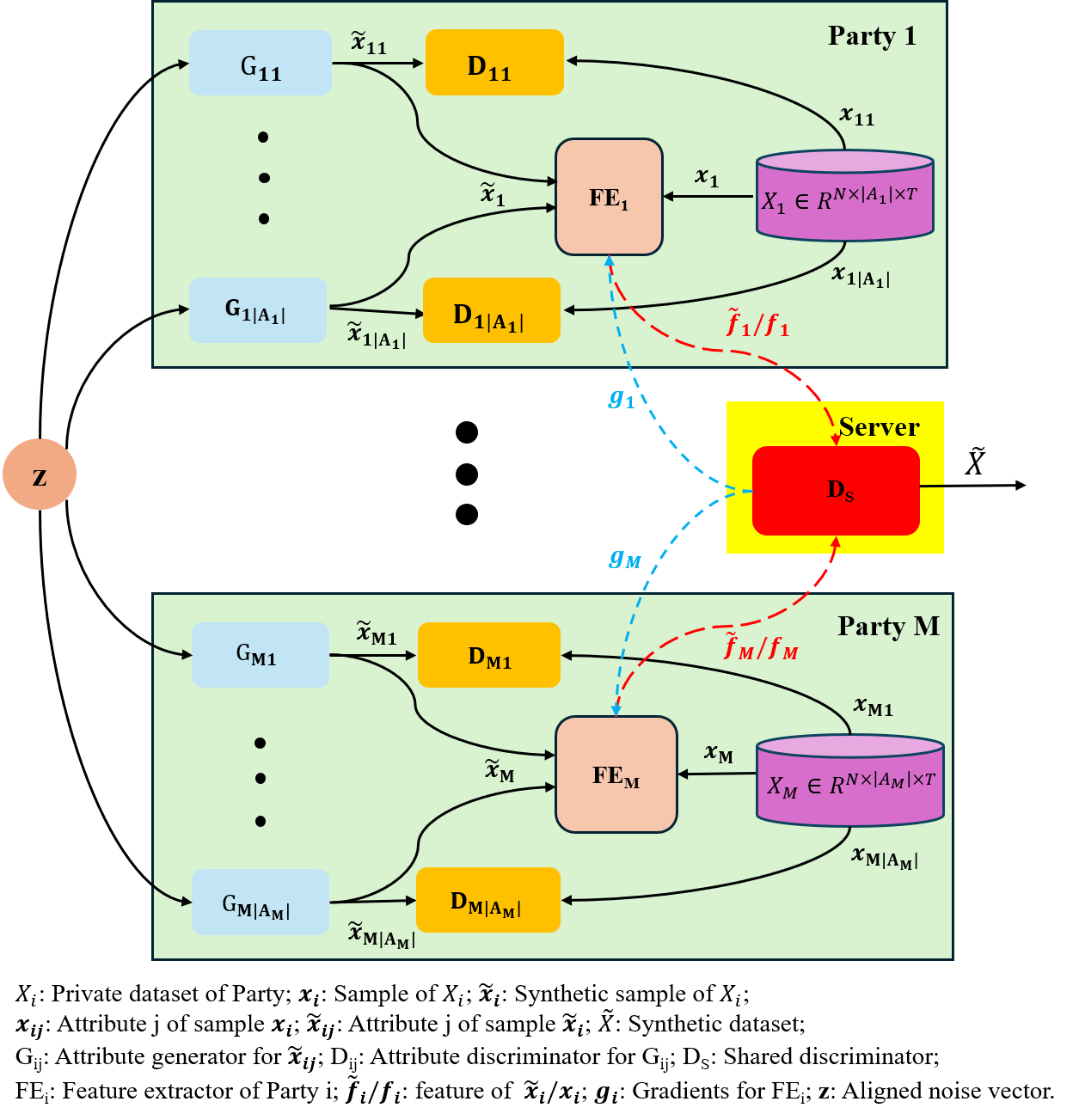}}
\caption{Framework of the Proposed VFLGAN-TS.}
\label{fig:system model}
\end{figure}

The framework of VFLGAN-TS is depicted in Fig. \ref{fig:system model}. VFLGAN-TS integrates the attribute discriminator of CosciGAN \cite{seyfi2022generating} and VFL \cite{liu2022vertical} to meet the objectives \eqref{eq:problem objective}. As shown in Fig. \ref{fig:system model}, each party in the framework maintains a local private dataset ($X_i$) with attribute set $A_i$, $|A_i|$ attribute generators, and $|A_i|$ attribute discriminators. In the local party $P_i$, each attribute generator, $G_{ij}$, is to generate a synthetic attribute ($\tilde{\boldsymbol{x}}_{ij}$), and the attribute discriminator, $D_{ij}$, is to optimize $G_{ij}$ to make the distribution of each synthetic attribute ($\tilde{\boldsymbol{x}}_{ij}$) to be similar to the distribution of the corresponding real attribute ($\boldsymbol{x}_{ij}$), i.e.,
\begin{equation}
    \label{eq:objective Dj}
     \min_{\theta_{G_{ij}}} \mathcal{D}(P_{G_{ij}(\boldsymbol{z})}, P_{\boldsymbol{x}_{ij}}).
\end{equation}
Notably, the random vector $\boldsymbol{z}$ is the same for all attribute generators, which can be achieved by a pseudorandom number generator. According to \eqref{eq:objective Dj}, the local discriminators are to learn the temporal correlation within one attribute. On the other hand, the shared discriminator in the sever is to learn the correlation among all attributes. Unlike CosciGAN, we use a trainable feature extractor ($FE$) in each party to extract features from the local attributes and send the features to the server, which can avoid transmitting private data to the server and protect privacy. Then, the features from all parties are concatenated in the server and discriminated by the shared discriminator ($D_S$) in the server. $D_S$ is to optimize all attribute generators simultaneously to make the distribution of synthetic data ($\tilde{\boldsymbol{x}}$) to be similar to that of real data ($\boldsymbol{x}$), i.e.,
\begin{equation}
\begin{aligned}
    \label{eq:objective DS}
    & \min_{\theta_{G}} \mathcal{D}(P_{\tilde{\boldsymbol{x}}}, P_{\boldsymbol{x}}), \\
    \text{where}\quad & \tilde{\boldsymbol{x}}=[G_{11}(\boldsymbol{z})\cdots,G_{M|A_1|}(\boldsymbol{z})], \\
    & G=\{G_{11}\cdots,G_{M|A_1|}\}.
\end{aligned}
\end{equation}
\textbf{Takeaway}: In the proposed VFLGAN-TS, the objectives of the vertically partitioned time-series data publication \eqref{eq:problem objective} are transformed to simultaneously optimize \eqref{eq:objective Dj} and \eqref{eq:objective DS}. Although the objectives of the shared discriminator can learn the correlation along both temporal and attribute dimensions, it mainly focuses on the attribute dimension since the local discriminators purely learn the correlation along the temporal dimension for each attribute. The separate treatment of the correlation along temporal and attribute dimensions is the key reason VFLGAN-TS fits the time-series data significantly better than VFLGAN \cite{Xun_vfl}.

\subsection{Training Process of VFLGAN-TS}

The generators in the VFLGAN-TS, which are deep neural networks, require gradient updates to effectively achieve the objectives of VFLGAN-TS in \eqref{eq:objective Dj} and \eqref{eq:objective DS} effectively.
However, since the above objectives are non-differential, they cannot be directly utilized to train the model. According to \eqref{eq:min-max objective}, we can convert the objectives \eqref{eq:objective Dj} into the following min-max optimization objectives to update the parameters of the attribute generator $G_{ij}$ and attribute discriminator $D_{ij}$,
\begin{equation}
\begin{aligned}
\label{eq:objective_local}
\min _{G_{ij}} \max _{D_{ij}} \mathbb{E}[\log (D_{ij}(\boldsymbol{x}_{ij}))]+\mathbb{E}[\log (1-D_{ij}(G_{ij}(\boldsymbol{z})))],
\end{aligned}
\end{equation}
where the subscript $ij$ means the model is responsible for Attribute $j$ of Party $i$, $\boldsymbol{x}_{ij}$ denotes the real Attribute $j$ of Party $i$, and $\boldsymbol{z}$ denotes a random noise. Similarly, to achieve objective \eqref{eq:objective DS}, all real and synthetic attributes should be considered. However, local attributes cannot be shared across parties due to the privacy constraints of the vertically partitioned scenario. To avoid sharing the raw data, we apply a local feature extractor in each party to process real and synthetic data and transmit the feature to the shared discriminator in the server, as shown in Fig. \ref{fig:system model}. The min-max optimization objectives between the attribute generators in all parties ($G$) and the shared discriminator ($D_S$) in the server can be described as,
\begin{equation}
\begin{aligned}
\label{eq:objective_server}
\min _{G} \max _{D_{S}} & \mathbb{E}[\log (D_{S}([\boldsymbol{f}_1, \cdots, \boldsymbol{f}_M]))] \\ 
+&\mathbb{E}[\log (1-D_{S}([\tilde{\boldsymbol{f}}_1, \cdots, \tilde{\boldsymbol{f}}_M]))], \\
\text{where} \quad \boldsymbol{f}_i = ET_i&(\boldsymbol{x}_i), \tilde{\boldsymbol{f}}_i = ET_i(\tilde{\boldsymbol{x}}_i),i \in \{1,\cdots,M\}.
\end{aligned}
\end{equation}
In \eqref{eq:objective_server}, $ET_i$ is the feature extractor in Party $i$, $\boldsymbol{x}_i$ is the real data of Party $i$, and $\tilde{\boldsymbol{x}}_i$ denotes the synthetic data of Party $i$.

According to \eqref{eq:objective_local}, the loss function for the local attribute discriminator ($D_{ij}$) can be described as,
\begin{equation}
\begin{aligned}
\label{eq:loss_local}
\mathcal{L}_{D_{ij}} = -(\mathbb{E}[\log (D_{ij}(\boldsymbol{x}_{ij}))]+\mathbb{E}[\log (1-D_{ij}(G_{ij}(\boldsymbol{z})))]).
\end{aligned}
\end{equation}
The loss function of $D_{ij}$ is the negative of objective \eqref{eq:objective_local} since the objective of $D_{ij}$ is to maximize \eqref{eq:objective_local} while gradient descent methods in deep learning framework are to optimize the model to minimize the loss function. Similarly, according to \eqref{eq:objective_server}, the loss function of the shared discriminator can be described as,
\begin{equation}
\begin{aligned}
\label{eq:loss_server}
\mathcal{L}_{D_S}  = & -(\mathbb{E}[\log (D_{S}([\boldsymbol{f}_1, \cdots, \boldsymbol{f}_M]))] \\ 
&+\mathbb{E}[\log (1-D_{S}([\tilde{\boldsymbol{f}}_1, \cdots, \tilde{\boldsymbol{f}}_M]))]), \\
\text{where} \quad \boldsymbol{f}_i =& ET_i(\boldsymbol{x}_i), \tilde{\boldsymbol{f}}_i = ET_i(\tilde{\boldsymbol{x}}_i),i \in \{1,\cdots,M\}.
\end{aligned}
\end{equation}
On the other hand, since the local attribute generator (e.g., $G_{ij}$) is required to learn the correlation along both temporal and attribute dimensions, the loss function of $G_{ij}$, $\mathcal{L}_{G_{ij}}$, should consider the objectives of local attribute discriminator and the shared discriminator and thus can be described as,
\begin{equation}
\begin{aligned}
\label{eq:loss_local_G}
\mathcal{L}_{G_{ij}} = & \mathbb{E}[\log (1-D_{ij}(G_{ij}(\boldsymbol{z})))] \\
&+\beta_1\mathbb{E}[\log (1-D_{S}([\tilde{\boldsymbol{f}}_1, \cdots, \tilde{\boldsymbol{f}}_M]))]),
\end{aligned}
\end{equation}
where $G_{ij}$ contributes to the $\tilde{\boldsymbol{f}}_i$ in the second term and $\beta$ is a balancing coefficient. Last, the loss function of the local feature extractor ($FE_i$) can be calculated with,
\begin{equation}
\begin{aligned}
\label{eq:loss_local_FE}
\mathcal{L}_{FE_{i}} = \beta_2\mathbb{E}[\log (1-D_{S}([\tilde{\boldsymbol{f}}_1, \cdots, \tilde{\boldsymbol{f}}_M]))]),
\end{aligned}
\end{equation}
since it is to learn the correlation along the attribute dimension and is optimized according to the loss function of $D_S$, and $\beta_2$ is a balancing coefficient.

According to the loss functions \eqref{eq:loss_local}, \eqref{eq:loss_server}, \eqref{eq:loss_local_G}, and \eqref{eq:loss_local_FE}, the gradients of parameters of the shared discriminator ($D_S$), local attribute discriminator ($D_{ij}$), local attribute generator ($G_{ij}$), and local feature extractor ($FE_i$), can be calculated by,
\begin{align}
    \label{eq:gradient_ds}
    &\mathcal{G}_{D_S} = \nabla_{\boldsymbol{\theta}_{D_S}}\mathcal{L}_{D_S}, \\
    \label{eq:gradient_dl}
    &\mathcal{G}_{D_{ij}} = \nabla_{\boldsymbol{\theta}_{D_{ij}}}\mathcal{L}_{D_{ij}}, \\
    \label{eq:gradient_gl}
    &\mathcal{G}_{G_{ij}} = \nabla_{\boldsymbol{\theta}_{G_{ij}}}\mathcal{L}_{G_{ij}}, \\
    \label{eq:gradient_fe}
    &\mathcal{G}_{FE_{i}} = \nabla_{\boldsymbol{\theta}_{FE_{i}}}\mathcal{L}_{FE_{i}}, \\
    \nonumber \text{where} \quad i\in &\{1,\cdots,M\},j \in \{1,\cdots,|A_i|\}.
\end{align}
After getting the gradients, we optimise the parameters by applying Adam \cite{kingma2014adam} optimizer.

The training process of VFLGAN-TS is summarized in Algorithm \ref{alg:Training DP-VFL-GAN}. In each training iteration, we first train the discriminators and feature extractors. We subsample a mini-batch of the aligned attributes in each party and generate a mini-batch of synthetic attributes with local attribute generators. The local attribute discriminators are trained to distinguish the real and synthetic attributes. The shared discriminator and feature extractors are trained to distinguish the combination of local real and synthetic attributes. Then, we train the attribute generators to generate more realistic synthetic attributes. The first and second steps repeat for $T_{max}$ iterations, and the algorithm outputs trained local attribute generators. 

\begin{algorithm}[htbp]
\caption{Training Process of (DP-)VFLGAN-TS}\label{alg:Training DP-VFL-GAN}
\footnotesize
\SetKwInOut{Input}{Input}
\SetKwInOut{Output}{Output}
\Input{$G_{ij}$: attribute generator for Attribute $j$ of Party $i$; $D_{ij}$: attribute discriminator for $G_{ij}$; $D_S$: shared discriminator; $FE_i$: feature extractor in Party $i$; $X_i$: dataset in Party $i$; $A_i$: attribute set of $X_i$; $M$: number of parties; $\mathcal{G}$: parameters gradients; $\mathcal{G}^1$: gradients of the first-layer parameters; $B$: batch size; $l$: latent dimension.}
\Output{Trained $\{G_{11},\cdots,G_{M|A_M|}\}$.}

Initialize all generators and discriminators; \\
\For{$epoch$ \text{in} $\{1,2,\cdots,T_{max}\}$}
{
    \tcp*[h]{\textbf{\textit{update discriminators (line 3 to 24)}}}\;
    {
        $\boldsymbol{x}^B=[\boldsymbol{x}_{1}^B,\cdots,\boldsymbol{x}_{M}^B]$ where $\boldsymbol{x}_{i}^B\subset X_i$ \tcp*[h]{\textit{Subsample a mini-batch of aligned data in each party}};\\
        
        Generate $\boldsymbol{z}^B \sim \mathcal{N}(0,1)^{l\times B}$ \\

        \tcp*[h]{\textit{Generate synthetic attribute}}\;
        \For{$i \in \{1,\cdots, M\}$ \text{and} $j \in \{1,\cdots,|A_i|\}$}
        {$\tilde{\boldsymbol{x}}_{ij}^B = G_{ij}(\boldsymbol{z}^B)$ }

        \tcp*[h]{\textit{Compute features}}\;
        \For{$i \in \{1,\cdots, M\}$}
        {$\tilde{\boldsymbol{f}}_{i}^B = FE_{i}([\tilde{\boldsymbol{x}}_{i1},\cdots,\tilde{\boldsymbol{x}}_{i|A_i|}])$\\
        $\boldsymbol{f}_{i}^B = FE_{i}([\boldsymbol{x}_{i1},\cdots,\boldsymbol{x}_{i|A_i|}])$ 
        }
        
        \tcp*[h]{\textit{Compute loss functions of discriminators}}\;
        \For{$i \in \{1,\cdots, M\}$ \text{and} $j \in \{1,\cdots,|A_i|\}$}
        {$\mathcal{L}_{D_{ij}} = -(\mathbb{E}[\log (D_{ij}(\boldsymbol{x}_{ij}))]+\mathbb{E}[\log (1-D_{ij}(G_{ij}(\boldsymbol{z})))])$}

        $\begin{aligned}
        \mathcal{L}_{D_S}  = & -(\mathbb{E}[\log (D_{S}([\boldsymbol{f}_1, \cdots, \boldsymbol{f}_M]))] \\ 
        &+\mathbb{E}[\log (1-D_{S}([\tilde{\boldsymbol{f}}_1, \cdots, \tilde{\boldsymbol{f}}_M]))])\\
        \end{aligned} $\\

        \tcp*[h]{\textit{Compute loss functions of feature extractors}}\;
        \For{$i \in \{1,\cdots, M\}$}
        {$\mathcal{L}_{FE_{i}} = \mathbb{E}[\log (1-D_{S}([\tilde{\boldsymbol{f}}_1, \cdots, \tilde{\boldsymbol{f}}_M]))])$}

        \tcp*[h]{\textit{Compute gradients of discriminators}}\;
        \For{$i \in \{1,\cdots, M\}$ \text{and} $j \in \{1,\cdots,|A_i|\}$}
        {$\mathcal{G}_{D_{ij}} = \nabla_{\boldsymbol{\theta}_{D_{ij}}}$ \\
        \uIf{Training a differentially private version}{
        $\mathcal{G}_{D_{ij}}^1 = clip(\mathcal{G}_{D_{ij}}^1,C)+\mathcal{N}\left(0, \sigma^2 (2C)^2 I\right)$} 
        }
        
        $\mathcal{G}_{D_S} = \nabla_{\boldsymbol{\theta}_{D_S}}\mathcal{L}_{D_S}$\\
        
        \tcp*[h]{\textit{Compute gradients of feature extractors}}\;
        \For{$i \in \{1,\cdots, M\}$ }
        {$\mathcal{G}_{FE_{i}} = \nabla_{\boldsymbol{\theta}_{FE_{i}}}$ \\
        \uIf{Training a differentially private version}{
        $\mathcal{G}_{FE_{i}}^1 = clip(\mathcal{G}_{FE_{i}}^1,C)+\mathcal{N}\left(0, \sigma^2 (2C)^2 I\right)$} 
        }

        Apply Adam optimizer with $\mathcal{G}_{D_{ij}}$, $\mathcal{G}_{D_{S}}$, and $\mathcal{G}_{FE_{i}}$ to update the parameters of $D_{ij}$, $D_S$, and $FE_i$, where $i\in\{1,\cdots,M\}$ and $j\in \{1,\cdots,|A_i|$.
    }
        
    \tcp*[h]{\textbf{\textit{update generators (line 25 to 35)}}}\;

    Generate $\boldsymbol{z}^B \sim \mathcal{N}(0,1)^{l\times B}$ \\
    
    \tcp*[h]{\textit{Generate synthetic attribute}}\;
    \For{$i \in \{1,\cdots, M\}$ \text{and} $j \in \{1,\cdots,|A_i|\}$}
    {$\tilde{\boldsymbol{x}}_{ij}^B = G_{ij}(\boldsymbol{z}^B)$ }

    \tcp*[h]{\textit{Compute features}}\;
    \For{$i \in \{1,\cdots, M\}$}
    {$\tilde{\boldsymbol{f}}_{i}^B = FE_{i}([\tilde{\boldsymbol{x}}_{i1},\cdots,\tilde{\boldsymbol{x}}_{i|A_i|}])$\\
    $\boldsymbol{f}_{i}^B = FE_{i}([\boldsymbol{x}_{i1},\cdots,\boldsymbol{x}_{i|A_i|}])$ 
    }
    \tcp*[h]{\textit{Compute loss functions of generators}}\;
    \For{$i \in \{1,\cdots, M\}$ \text{and} $j \in \{1,\cdots,|A_i|\}$}
    {$\begin{aligned}
    \mathcal{L}_{G_{ij}} &= \mathbb{E}[\log (1-D_{ij}(G_{ij}(\boldsymbol{z})))] \\
    &+\beta\mathbb{E}[\log (1-D_{S}([\tilde{\boldsymbol{f}}_1, \cdots, \tilde{\boldsymbol{f}}_M]))])
    \end{aligned}$}

    \tcp*[h]{\textit{Compute gradients of generators}}\;
    \For{$i \in \{1,\cdots, M\}$ \text{and} $j \in \{1,\cdots,|A_i|\}$}
    {$\mathcal{G}_{G_{ij}} = \nabla_{\boldsymbol{\theta}_{G_{ij}}}$
    }

    Apply Adam optimizer with $\mathcal{G}_{G_{ij}}$ to update the parameters of $G_{ij}$, where $i\in\{1,\cdots,M\}$ and $j\in \{1,\cdots,|A_i|$.
    
}
Return $\{G_{11},\cdots,G_{M|A_M|}\}$.
\end{algorithm}

\subsection{Differentially Private VFLGAN-TS}
The training process of differentially private VFLGAN-TS (DPVFLGAN-TS) is also summarized in Algorithm \ref{alg:Training DP-VFL-GAN} since the training processes of both versions are the same for most steps. There are two main differences between the two training processes. (i) We adapt the Gaussian mechanism \eqref{eq:add noise} proposed in our previous study \cite{Xun_vfl} to train DPVFLGAN-TS. When updating the parameters of the first linear layers of the local attribute discriminators and feature extractors, we clip and add Gaussian noise to the gradients as follows,
\begin{align}
\label{eq:clip}
    &clip(\mathcal{G}_{model}^1,C) = \mathcal{G}_{model}^1 / \max \left(1,\left\|\mathcal{G}_{model}^1\right\|_2 / C\right), \\
\label{eq:add noise}
    &\mathcal{G}_{model}^1 = clip(\mathcal{G}_{model}^1,C) + \mathcal{N}\left(0, \sigma^2 (2C)^2 I\right),
\end{align}
where the subscript $model$ denotes any local attribute discriminator and feature extractor, $\mathcal{G}_{model}^1$ denotes the gradients of the first layer parameters of the $model$, and $C$ denotes the clipping bound. In Theorem \ref{theorem RDP guarantee} we prove that clipping and adding Gaussian noise to the gradients of the first-layer parameters can provide a DP guarantee for the local attribute discriminators, attribute generators, and feature extractors. (ii) We need to set a privacy budget, i.e., $(\epsilon, \delta)$-DP, before training DPVFLGAN-TS. Here, we apply the official implementation of \cite{wang2019subsampled} to select the proper training iterations ($T_{max}$ in Algorithm \ref{alg:Training DP-VFL-GAN}) and $\sigma$ in \eqref{eq:add noise} to achieve the privacy budget, i.e., $(\epsilon, \delta)$-DP. 

\begin{myTheo} \label{theorem RDP guarantee}
    (RDP Guarantee) All the local attribute discriminators and feature extractors satisfy ($\alpha$, $\alpha/(2\sigma^2)$)-RDP and the local attribute generators satisfy ($\alpha$, $\alpha/\sigma^2$)-RDP in one training iteration of DPVFLGAN-TS.
\end{myTheo}
\begin{proof}
The proof of Theorem \ref{theorem RDP guarantee} is presented in \textbf{Appendix I}.
\end{proof}

Now, we introduce a method to select appropriate $\sigma$ and $T_{max}$ to meet the DP budget using Theorem \ref{theorem RDP guarantee}. First, the RDP guarantee in Theorem \ref{theorem RDP guarantee} can be enhanced by Proposition \ref{theorem: subsample RDP} for external attackers and the server since we subsample mini-batch records from the whole training dataset in Algorithm \ref{alg:Training DP-VFL-GAN}. Then, the RDP budget is accumulated by $T_{max}$ iterations, which can be calculated according to Proposition \ref{Prop: RDP composition}. Last, the RDP guarantee is converted to DP guarantee using \eqref{eq:RDP to DP}. We can adjust the $\sigma$ and $T_{max}$ to make the calculated DP guarantee meet our DP budget. 
Notably, similar to the DP guarantee in \cite{Xue01,Xun_vfl}, the above DP guarantee specifically addresses external threats. This is due to the deterministic nature of mini-batch selection in each training iteration for internal parties (excluding the server), as opposed to a subsampling process. For internal adversaries, privacy can be enhanced by sacrificing efficiency. For example, the mini-batch size can be changed to $\hat{B}>B$. When updating the parameters, each party can randomly select the gradients of $B$ samples and mask the gradients of other ($\hat{B}-B$) samples. In this way, all parties do not know precisely which samples others use to update the parameters. 

\subsection{Privacy Auditing for VFLGAN-TS}

In the section, we enhance the ASSD described in Section \ref{sec:pre_ASSD} since it is not effective for the synthetic datasets generated by VFLGAN-TS, as shown in Table \ref{tab:ASSD}. We enhance ASSD in two directions: target sample selection and feature extraction. We apply the same threat model as \cite{Xun_vfl}; details are available in \textbf{Appendix II}.

We propose two methods to select the target sample. The first method is to choose the sample with the maximum nearest neighbour distance, i.e.,
\begin{equation}
    \label{eq:maximum minimum distance}
    \mathop{\arg\max}_{\boldsymbol{x}\in X} \ \min_{\boldsymbol{x^\prime}\in X \setminus \boldsymbol{x}}||\boldsymbol{x} - \boldsymbol{x^\prime}||_n,
\end{equation} 
where $||\cdot||_n$ denotes the $n$-norm function. In the second method, we first construct a candidate set $S_c$ by selecting $m$ samples with the first $m$ maximum nearest neighbour distances,
\begin{equation}
    \label{eq:m-maximum minimum distance}
    S_c = \text{Top-$m$} \{\min_{\boldsymbol{x^\prime}\in X \setminus \boldsymbol{x}}||\boldsymbol{x} - \boldsymbol{x^\prime}||_n\}.
\end{equation} 
Then, we train a VFLGAN-TS with the whole training dataset and use it to generate a synthetic dataset $\tilde{X}$. Then, the target sample is selected by,
\begin{equation}
    \label{eq:kdd}
   \mathop{\arg\min}_{\boldsymbol{x}\in S_c} \ \operatorname{sum}(\text{minimum-$k$} \{\min_{\boldsymbol{x^\prime}\in \tilde{X}}||\boldsymbol{x} - \boldsymbol{x^\prime}||_n\}),
\end{equation} 
where $\operatorname{sum}(S)$ denotes sum of all elements in set $S$.
\begin{tcolorbox}
\textbf{Takeaway:} \eqref{eq:maximum minimum distance} is to select the most out-of-distribution sample while \eqref{eq:kdd} is to select the most influential sample for the synthetic dataset. Some papers \cite{Theresa01,meeus2023achilles} assume that the most out-of-distribution samples are the most vulnerable to privacy attacks. However, we found that in some cases, the most out-of-distribution sample only has a minor influence on the generative model and is thus hard to attack.
\end{tcolorbox}

According to \eqref{eq:kdd}, we propose the KNN feature, the sum of the $k$ nearest neighbour distances between the target sample ($x_t$) and samples of the synthetic dataset ($\tilde{X}$), 
\begin{equation}
    \label{eq:kdd feature}
   F_{\text{KNN}} = \operatorname{sum}(\text{minimum-$k$} \{\min_{\boldsymbol{x^\prime}\in \tilde{X}}||\boldsymbol{x_t} - \boldsymbol{x^\prime}||_n\}).
\end{equation} 
For the KNN feature, we compute the AUC-ROC in line 4 of Algorithm \ref{alg:MI adversary} to estimate the privacy breaches through synthetic datasets instead of training a random forest like described in \cite{Xun_vfl}. This is because the KNN feature is a single value and AUC-ROC is better for representing the difference between $F_{0_{1:M}}$ and $F_{1_{1:M}}$ in Algorithm \ref{alg:MI adversary} than splitting the feature into training and test sets and training a random forest on the training set and evaluating the accuracy of the random forest on the test set.

\section{Experiments and Results}
This section first introduces the experimental environment, then presents the evaluation results of VFLGAN-TS, and last, gives a privacy analysis of the proposed framework and generated synthetic datasets.

\subsection{Experimental Environments}
\subsubsection{Datasets} We conduct experiments with two synthetic datasets and three real-world datasets. 

\textbf{Synthetic Sine Datasets}: Similar to \cite{seyfi2022generating,jeon2022gt}, the two synthetic datasets are constructed with sine signals. The first synthetic dataset is named two-attribute Sine dataset, which is constructed as follows,
\begin{equation}
\label{eq:sine data 1}
    F_{11} = A sin(2\pi f_1 \bm t)+ \epsilon, \quad F_{21} = A sin(2\pi f_2 \bm t)+ \epsilon,  
\end{equation}
where $F_{ij}$ denotes the attribute $j$ stored in Party $i$, $A$ can be sampled from $\mathcal{N}(0.4,0.05)$ or $\mathcal{N}(0.6,0.05)$ with equal propability, $\epsilon\in\mathcal{N}(0,0.05)$, and $f_1$ and $f_2$ are set 0.01 and 0.005, respectively. In summary, in the two-party scenario, each party store one attribute, and the amplitudes of two attributes are the same for a given sample. We can utilize this characteristic to evaluate whether the generative methods learn the correlation between the two attributes stored in different parties. We generate 1,024 samples for both class, i.e., $A\in\mathcal{N}(0.4,0.05)$ and $A\in\mathcal{N}(0.6,0.05)$, and each attribute consists of 800 time steps ($\bm t = [0,1, \cdots, 799]$ in \eqref{eq:sine data 1}).

The second synthetic dataset, named six-attribute Sine dataset, also contains two classes of samples and six attributes to evaluate the scalability of our methods. The construction details can be found in \textbf{Appendix III}. In the two-party scenario, each party holds three attributes. Similar to \eqref{eq:sine data 1}, the frequency of each attribute is different and the amplitudes of the six attributes are the same for a given sample. 

\textbf{EEG Eye State Dataset}: EEG dataset \cite{misc_eeg_eye_state_264} contains 14 attributes and a label indicating whether the patient’s eyes were open or closed. Similar to \cite{seyfi2022generating}, we select five attributes that achieve the best performance on the classification task and split the datasets into two datasets according to the label, i.e., EEG 0 and EEG 1. In the two-party scenario, one party has three attributes, and the other has two.

\textbf{Stock Dataset}: Stock dataset consists of Google stock price data from 2004 to 2019 and includes six attributes. We apply the same preprocessing method as \cite{yoon2019time} and obtain a dataset containing 3661 frames, and each frame contains six attributes with 24 time steps. In the two-party scenario, each party has three attributes.

\textbf{Energy Dataset}: Energy dataset \cite{misc_appliances_energy_prediction_374} contains 19712 frames, and each frame consists of 28 attributes with 24 time steps. In the two-party scenario, each party has 14 attributes.

\subsubsection{Evaluation Metrics} In this paper, we apply Wasserstein distance, performance on downstream tasks, and visualization for evaluation.

\textbf{Wasserstein distance}: The Wasserstein distance measures the similarity between two probability distributions \cite{wiki_wd}. A lower Wasserstein distance means there is more similarity between the two distributions. We calculate the Wasserstein distance between the attributes of the synthetic and real datasets to evaluate the performance of different methods. Assuming there are $N$ attributes and each attribute is a time series with $T$ time steps, we calculate the Average Wasserstein Distance (AWD) as,
\begin{equation}
\label{eq:AWD}
\text{AWD} = \frac{1}{TN}\sum_{t=1}^T\sum_{n=1}^N \operatorname{WD}(F^t_n,\tilde{F}^t_n),
\end{equation}
where $F^t_n$ and $\tilde{F}^t_n$ represent the combination of attribute $n$ at time step $t$ across all samples in the real and synthetic datasets, respectively, and $\operatorname{WD}$$\footnote{We use SciPy library to compute the Wasserstein distance and details can be found in \url{https://docs.scipy.org/doc/scipy/reference/generated/scipy.stats.wasserstein_distance.html\#scipy.stats.wasserstein_distance}}$ is a function to estimate the Wasserstein distance. Although AWD cannot measure the correlation among different attributes or time steps, it has two benefits. First, its rapid computation allows the evaluation after each training epoch, and we can thus select the model with the minimum AWD as the final model. Second, it is privacy-preserving since it requires no information exchange between different parties.


\textbf{Performance on Downstream Tasks}: Performance on Downstream Tasks (PDT) is to assess the prospects of synthetic datasets in real-world applications. Different from \cite{seyfi2022generating,jeon2022gt} and \cite{yoon2019time} that only consider the training on synthetic data and testing on real data scenarios, we consider the following four scenarios like \cite{Xun_vfl}: (i) training on real data and testing on real data (TRTR), (ii) training on synthetic data and testing on synthetic data (TSTS), (iii) training on real data and testing on synthetic data (TRTS), (iv) and vice versa (TSTR). The TRTR setting serves as a baseline for PDT. The performance in TSTS, TRTS, and TSTR scenarios should be similar to that in the TRTR scenario if the distribution of synthetic data is similar to that of real data. Thus, we use Total Performance Difference (TPD) as the final comparison metric,
\begin{equation}
    \label{eq:TPD}
    \text{TPD}=\sum_{i\in \{TSTS,TRTS,TSTR\}}(P_i - P_{TRTR}),
\end{equation}
where $P$ denotes performance.

\textbf{Visualization}: The visualization results can intuitively show how closely the distribution of generated samples resembles that of the real samples in two-dimensional space. For visualization, we project synthetic and real datasets into a two-dimensional space through t-SNE \cite{van2008visualizing} and PCA \cite{bryant1995principal} analysis.

\begin{figure} 
  \centering
  \subfloat[Two-attribute Sine]
  {\includegraphics[width=0.5\linewidth]{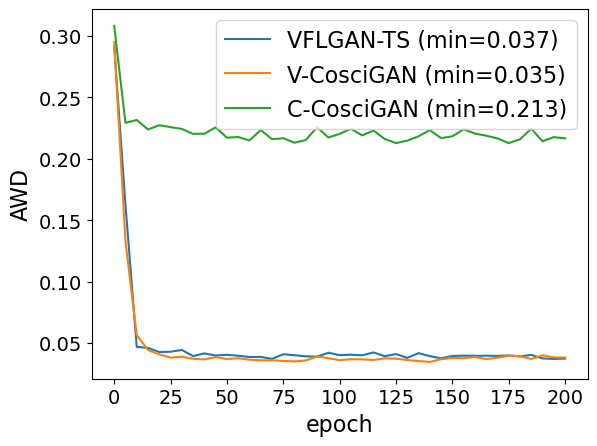}\label{fig:wd curve sine_2}}     
  \subfloat[Six-attribute Sine]
  {\includegraphics[width=0.5\linewidth]{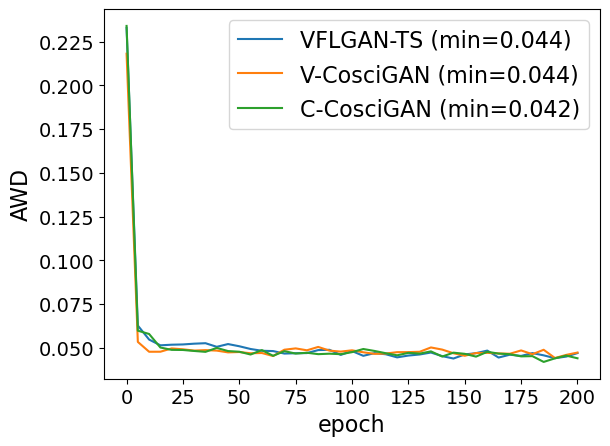}\label{fig:wd curve sine_6}}
  \caption{Wasserstein distance curves during training different methods on Sine Datasets.}
  \label{fig:wd curve sine}
\end{figure}

\subsubsection{Baselines} We compare VFLGAN-TS with VFLGAN \cite{Xun_vfl} to show the improvement. We also consider two CosciGAN-based baselines to evaluate the proposed VFLGAN-TS since VFLGAN-TS is constructed based on CosciGAN. The first baseline is the CosciGAN trained in a centralized manner, i.e., the model can access all attributes of the dataset, which is named Centralized CosciGAN (C-CosciGAN). The other baseline is the CosciGAN trained in a vertically partitioned manner, i.e., which is named Vertical CosciGAN (V-CosciGAN). In V-CosciGAN, each party trains a local CosciGAN that can only access the local attributes. In the inference process, the synthetic dataset is constructed with the synthetic attributes generated by each local CosciGAN. 

\begin{tcolorbox}
\textbf{Takeaway}: C-CosciGAN can represent the upper limit of VFLGAN-TS since it is trained centrally. By comparing VFLGAN-TS with C-CosciGAN, we can evaluate the performance degradation associated with the vertically partitioned scenario. On the other hand, C-CosciGAN can represent the lower limit of VFLGAN-TS since the generators are trained locally. Comparing VFLGAN-TS to V-CosciGAN allows us to assess the performance gains attributed to the VFL framework over a straightforward adaptation of CosciGAN to the vertically partitioned scenario. Besides, for a fair comparison, the structure of the attribute generators and attribute discriminators are the same for VFLGAN-TS, VFLGAN, C-CosciGAN, and V-CosciGAN. Also, we strive to align the structure and computation of the shared discriminators in VFLGAN and VFLGAN-TS with that of the central discriminators in C-CosciGAN and V-CosciGAN.
\end{tcolorbox}

\subsection{Evaluation Results} \label{sec:evaluation results}
This section shows the evaluation results based on the above-mentioned experimental environment.

\subsubsection{Sine Datasets}

\begin{table}[htbp]
\renewcommand\arraystretch{1.}
    \centering
    \caption{MAE between Synthetic Datasets and Real Datasets}
    \begin{tabular}{c|c|c|c}
    \hline
        {Dataset} & \textbf{V-CosciGAN} & \textbf{C-CosciGAN}& \textbf{VFLGAN-TS}\\
         \hline
         {Two-feature Sine} &\textcolor{red}{\textbf{0.046}} &0.215 &0.049\\
         \hline
         {Six-feature Sine} &\textcolor{red}{\textbf{0.050}} &\textcolor{red}{\textbf{0.050}} & 0.051 \\
         \hline         
    \end{tabular}
    \label{tab:mae of sine}
\end{table}

\begin{table*}
\centering
{
\renewcommand\arraystretch{1.0}
    \caption{Performance on Downstream Classification Task}
    \label{tab:classification_eeg}
    \begin{tabular}{ccccccccc|cc}

    \hline
    {} & \multicolumn{2}{c}{\textbf{TRTR}} & \multicolumn{2}{c}{\textbf{TSTS}} & \multicolumn{2}{c}{\textbf{TRTS}} & \multicolumn{2}{c}{\textbf{TSTR}} & \multicolumn{2}{c}{\textbf{TPD}}\\
    \textbf{Method} & \textbf{Acc} & \textbf{F1} & \textbf{Acc} & \textbf{F1} & \textbf{Acc} & \textbf{F1} & \textbf{Acc} & \textbf{F1}  & \textbf{Acc} & \textbf{F1} \\
     \hline
     {\textbf{V-CosciGAN}} &0.92(0.019) &0.92(0.018) &1.00(0.001) &1.00(0.001) &0.68(0.029) &0.68(0.026) &0.68(0.025) &0.69(0.031) & 0.56 & 0.55\\     
     \hline
     \textbf{C-CosciGAN} &0.92(0.019) &0.92(0.018) &1.00(0.002) &1.00(0.002) &0.87(0.024) &0.87(0.026) &0.74(0.018) &0.75(0.026) & \textcolor{red}{\textbf{0.31}} & \textcolor{red}{\textbf{0.30}}\\  
     \hline
     
     \textbf{VFLGAN \cite{Xun_vfl}} &0.92(0.019) &0.92(0.018) &1.00(0.000) &1.00(0.000) &0.85(0.022) &0.84(0.026) &0.68(0.039) &0.70(0.053) & 0.39 & 0.38\\
     \hline

     \textbf{VFLGAN-TS} &0.92(0.019) &0.92(0.018) &1.00(0.003) &1.00(0.003) &0.86(0.015) &0.86(0.017) &0.74(0.016) &0.74(0.026) & 0.32 & 0.32\\
     \hline

    \end{tabular}
    }
    \begin{tablenotes}
        \footnotesize
        \item\textbf{TRTR}: Train on real test on real; \textbf{TSTS}: train on synthetic test on synthetic; \textbf{TRTS}: train on real test on synthetic; \textbf{TSTR}: train on synthetic test on real; \textbf{Ac}: accuracy; \textbf{F1}: F1-score; \textbf{TPD}: Total performance difference is the final comparison metric and \textbf{lower is better}.
    \end{tablenotes}
\end{table*}

Given a sample from Sine Datasets, the amplitude of each feature is the same. So, we exploit this property to evaluate the quality of Sine datasets. Different from \eqref{eq:AWD}, to calculate Wasserstein Distance (WD), we first estimate the amplitude vector ($\bm a$) of each sample in real and synthetic datasets by,
\begin{equation}
\label{eq:amplitude estimation}
    \bm a \triangleq [a_1, \cdots, a_N] = [\mathcal{A}(\bm f_1),\cdots,\mathcal{A}(\bm f_N)],
\end{equation}
where $a_1$ denotes the estimated amplitude of feature $f_1$ and $\mathcal{A}$ denotes the estimation function. Then, we calculate the AWD of the amplitude distributions of each feature between real and synthetic datasets by,
\begin{equation}
\label{eq:WD of amplitude estimation}
    \text{AWD} = \sum_{n=1}^N\operatorname{WD}(A_n,\tilde{A}_n),
\end{equation}
where $A_n$ is the vector that combines the amplitudes $a_i$ of all samples in the real dataset and $\tilde{A}_n$ is the vector that combines the amplitudes $\tilde{a}_n$ of all samples in the synthetic dataset. We do not need to calculate the average Wasserstein distance across the time dimension like \eqref{eq:AWD} since the time-series features have been converted to amplitude features. 
Figure \ref{fig:wd curve sine} shows the AWD curves of the three methods during the training process on two-feature and six-feature Sine datasets. As shown in Fig. \ref{fig:wd curve sine}, the C-CosciGAN performs poorly on the two-feature Sine dataset while the proposed VFLGAN-TS performs similarly to V-CosciGAN. Besides, VFLGAN-TS, C-CosciGAN, and V-CosciGAN show similar performance on the six-feature Sine dataset.

\begin{tcolorbox}
\textbf{Takeaway}: AWD for amplitude features primarily assesses distribution similarity along the temporal dimension. In the two-feature Sine dataset, C-CosciGAN may focus excessively on feature correlations, as illustrated in Fig. \ref{fig:dist_coscigan}, which hampers its ability to capture temporal correlations accurately. On the contrary, V-Coscigan exhibits the best performance w.r.t. AWD since it solely focuses on the temporal dimension and does not need to learn the correlation between the two sine signals, as shown in Fig. \ref{fig:dist_withoutSD}. In contrast, VFLGAN-TS achieves the best trade-off.
\end{tcolorbox}

\begin{table*}
\centering
{
\renewcommand\arraystretch{1.0}
    \caption{Mean Absolute Error (MAE) on Downstream Forecasting Task}
    \label{tab:forecasting_stock&energy}
    \begin{tabular}{c|ccccc|c}

    \hline
    \textbf{Dataset} &\textbf{Method} & {\textbf{TRTR}} & {\textbf{TSTS}} & {\textbf{TRTS}} & {\textbf{TSTR}} &{\textbf{TPD}}\\

     \hline
     \multirow{4}*{\textbf{Stock}} & \textbf{V-CosciGAN} & 0.050 (0.005) & 0.042 (0.001) &  0.063 (0.005) & 0.054 (0.002) & 0.025\\     
     \cline{2-7}
     {} & \textbf{C-CosciGAN}  & 0.050 (0.005) &0.048 (0.005) & 0.057 (0.006) & 0.056 (0.005) & 0.015\\  
     \cline{2-7}
     {} & \textbf{VFLGAN \cite{Xun_vfl}}  & 0.050 (0.005) &0.042 (0.004) & 0.046 (0.002) & 0.053 (0.004) &0.015\\
     \cline{2-7}
     {} & \textbf{VFLGAN-TS} & 0.050 (0.005) &0.048 (0.003) & 0.050 (0.005) & 0.050 (0.003) &\textcolor{red}{\textbf{0.002}}\\
     \hline
     
     \hline
     \multirow{4}*{\textbf{Energy}} & \textbf{V-CosciGAN}  &0.062 (0.0027) &0.063 (0.0034) &0.122 (0.0025) &0.099 (0.0052) & 0.098\\
     \cline{2-7}
     {} & \textbf{C-CosciGAN} &0.062 (0.0027) &0.061 (0.0025)& 0.064 (0.0026) & 0.066 (0.0020) &\textcolor{red}{\textbf{0.007}}\\  
     \cline{2-7}
     {} & \textbf{VFLGAN \cite{Xun_vfl}} & \multicolumn{5}{c}{Not available since VFLGAN did not converge}\\
     \cline{2-7}
     {} & \textbf{VFLGAN-TS} &0.062 (0.0027) & 0.060 (0.0034) & 0.061 (0.0028) & 0.067 (0.0032) &0.008\\
     \hline

    \end{tabular}
    }
    \begin{tablenotes}
        \footnotesize
        \item Abbreviations are the same as Table \ref{tab:classification_eeg}. Experiments repeat ten times. Average MAE is shown in this table and standard deviation is shown in brackets. 
    \end{tablenotes}
     \vspace{-2mm}
\end{table*}

After training, we select the parameters that achieve minimal AWD in Fig. \ref{fig:wd curve sine} as the final model to generate synthetic datasets. We use the amplitude in Fig. \ref{fig:wd curve sine} to estimate the temporal similarity. We now estimate the temporal similarity by the mean absolute error (MAE). For a given synthetic sample ($\tilde{F}_m$), we first estimate the amplitude vector $a$ according to \eqref{eq:amplitude estimation}. According to the construction of Sine datasets, the ground truth of feature $i$ ($F_{mi}$) can be estimated as $ F_{mi} = a_i sin(2\pi f_i \bm t)$. Then, the MAE can be calculated by, 
\begin{equation}
\label{eq:MAE sine}
    \text{MAE} = \sum_{m=1}^M\sum_{n=1}^N\sum_{t=1}^T|\tilde{F}_{mn}^t - F_{mn}^t|.
\end{equation}
As shown in Table \ref{tab:mae of sine}, for the two-feature Sine dataset, V-CosciGAN achieves the best performance while C-CoscigGAN presents the worst performance, and for the six-feature Sine dataset, the three methods present similar performance. The MAE and AWD give consistent experimental results, which shows that AWD is an effective metric for measuring temporal similarity. In the case of real-world datasets, MAE is impossible to estimate. Thus, we apply AWD in the following experiments to select the best parameters.

\begin{figure}[ht]
  \centering
  \subfloat[{Real Dataset}]
  {\includegraphics[width=0.5\linewidth]{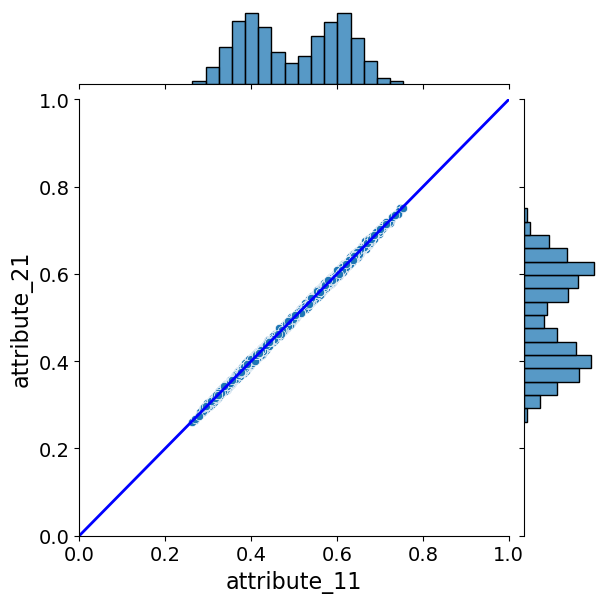}\label{fig:dist_real}}     
  \subfloat[V-CosciGAN]
  {\includegraphics[width=0.5\linewidth]{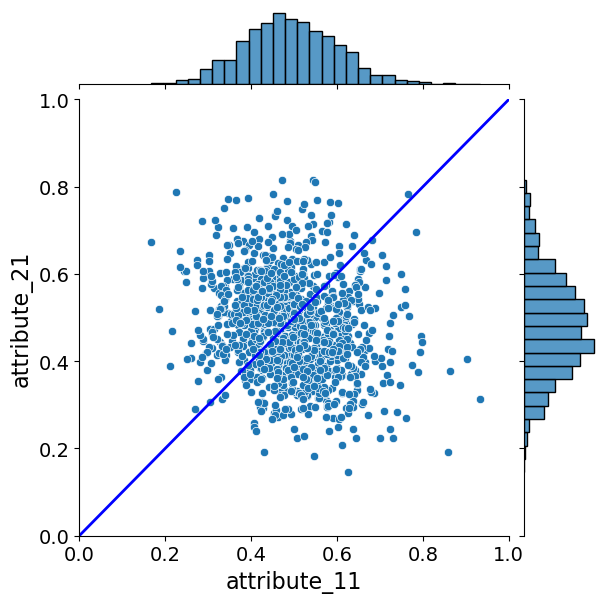}\label{fig:dist_withoutSD}} \\
  \subfloat[C-CosciGAN]
  {\includegraphics[width=0.5\linewidth]{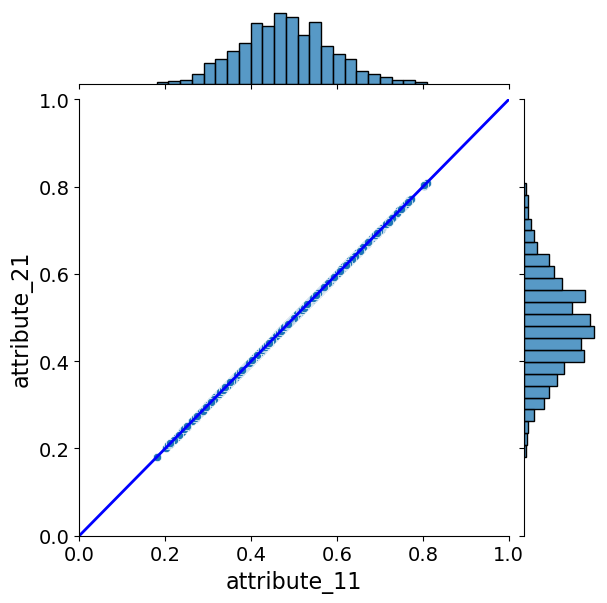}\label{fig:dist_coscigan}}
  \subfloat[VFLGAN-TS]
  {\includegraphics[width=0.5\linewidth]{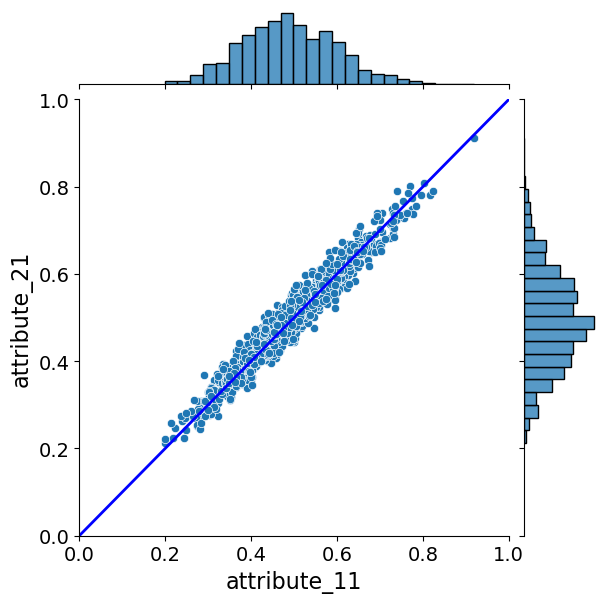}\label{fig:dist_withSD}}
  \caption{The histograms in each sub-figure show the amplitude distribution of each attribute. The spots in each sub-figure represent a sample's amplitudes of both attributes.}
  \label{fig:dist_sine}
  \vspace{-4mm}
\end{figure}

Figure \ref{fig:dist_sine} shows the distribution of feature amplitude and the ratio between the amplitudes of feature 1 and feature 2. As shown in Fig. \ref{fig:dist_real}, the ratio between the amplitudes of feature 1 and feature 2 should be around 1.0 for the samples of the real dataset. As shown in Fig. \ref{fig:dist_coscigan}, C-CosciGAN learns this property perfectly. As shown in Fig. \ref{fig:dist_withSD}, VFLGAN-TS can also learn this property, but the ratio has a larger range compared to C-CosciGAN due to the information loss in the vertically partitioned scenario. On the other hand, VFLGAN-TS achieves significant improvement compared to V-CosciGAN, which cannot learn the correlation of the features located in different parties.

\begin{tcolorbox}
\textbf{Takeaway}: From the evaluation results on Sine datasets, we can conclude that C-CosciGAN performs better in learning the correlation between different features but may perform poorly in learning temporal distribution. On the other hand, VFLGAN-TS performs better in learning the temporal distribution compared to C-CosciGAN and better in learning feature correlation than V-CosciGAN. 
\end{tcolorbox}

\begin{figure}
    \centering
    \subfloat[EEG 0]
  {\includegraphics[width=0.5\linewidth]{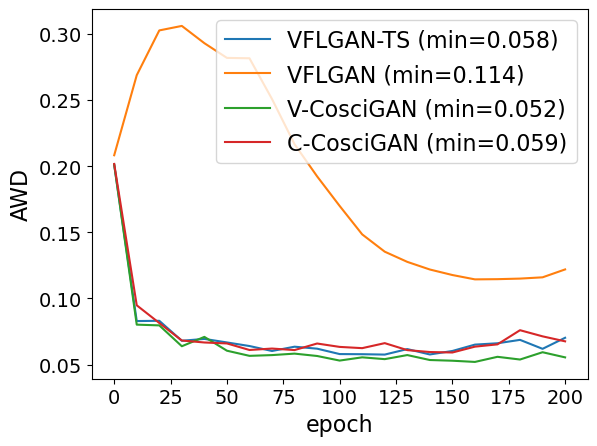}\label{fig:wd_curve_eeg_0}}     
  \subfloat[EEG 1]
  {\includegraphics[width=0.5\linewidth]{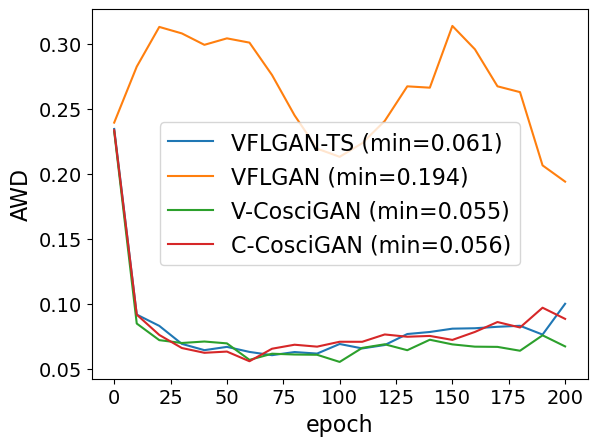}\label{fig:wd_curve_eeg_1}}
    \caption{Wasserstein distance curves during training different methods on EEG Dataset.}
    \label{fig:wd_curve_eeg}
    \vspace{-6mm}
\end{figure}

\subsubsection{EEG Dataset} For the EEG dataset, we first use AWD \eqref{eq:AWD} to evaluate the similarity of each attribute at each time step between real and synthetic datasets. 
As shown in Fig. \ref{fig:wd_curve_eeg}, the three methods perform similarly in learning the temporal distribution. However, VFLGAN \cite{Xun_vfl} performs poorly and cannot even converge for EEG 1.



\begin{figure*} 
  \centering
  \subfloat[{V-CosciGAN (PCA)}]
  {\includegraphics[width=0.25\linewidth]{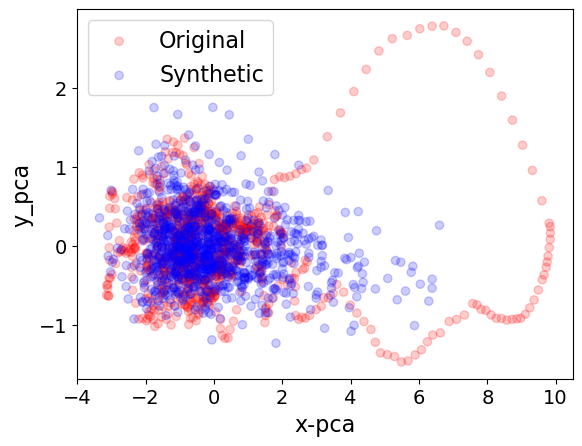}\label{fig:label_0_withoutSD_PCA}}     
  \subfloat[{V-CosciGAN (t-SNE)}]
  {\includegraphics[width=0.25\linewidth]{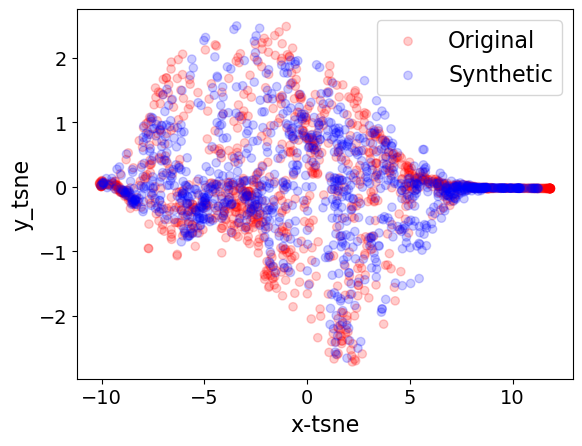}\label{fig:label_0_withoutSD_tSNE}} 
  \subfloat[{V-CosciGAN (PCA)}]
  {\includegraphics[width=0.25\linewidth]{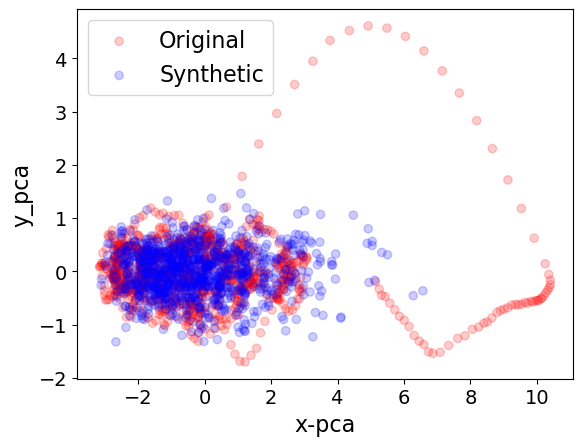}\label{fig:label_1_withoutSD_PCA}}     
  \subfloat[{V-CosciGAN (t-SNE)}]
  {\includegraphics[width=0.25\linewidth]{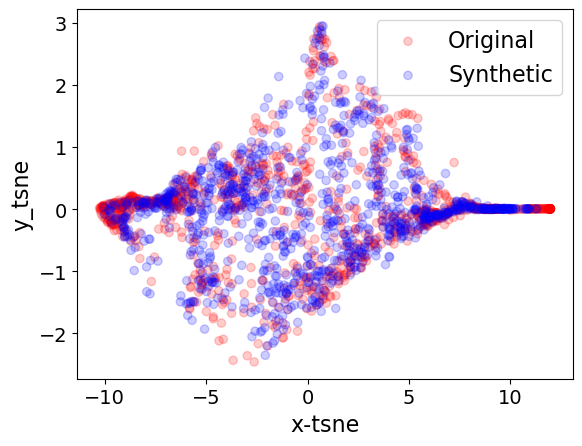}\label{fig:label_1_withoutSD_tSNE}}\\
  \subfloat[{C-CosciGAN (PCA)}]
  {\includegraphics[width=0.25\linewidth]{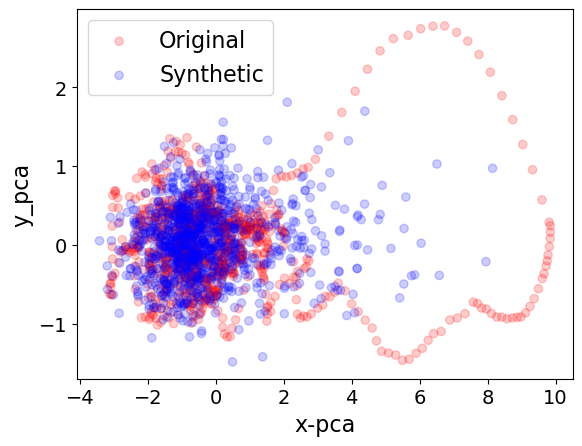}\label{fig:label_0_coscigan_PCA}}     
  \subfloat[{C-CosciGAN (t-SNE)}]
  {\includegraphics[width=0.25\linewidth]{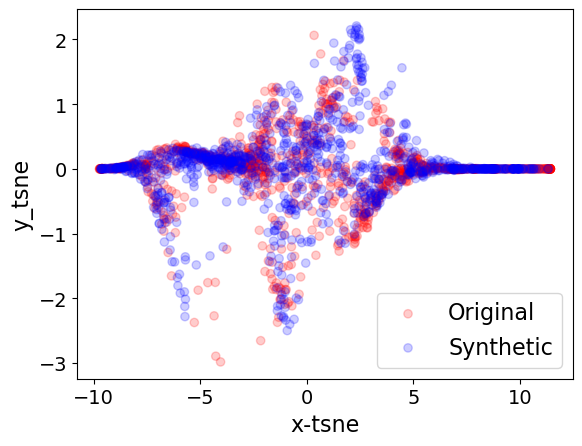}\label{fig:label_0_coscigan_tSNE}} 
  \subfloat[{C-CosciGAN (PCA)}]
  {\includegraphics[width=0.25\linewidth]{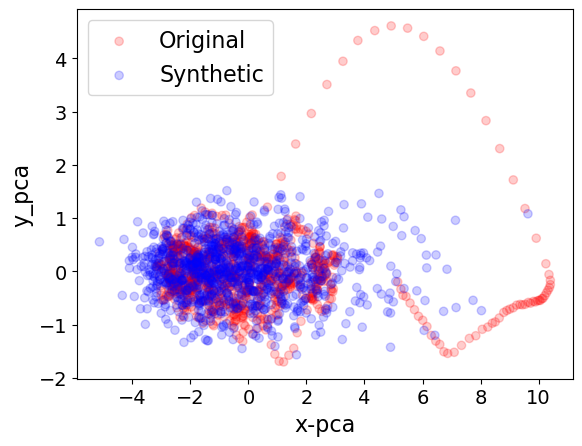}\label{fig:label_1_coscigan_PCA}}     
  \subfloat[{C-CosciGAN (t-SNE)}]
  {\includegraphics[width=0.25\linewidth]{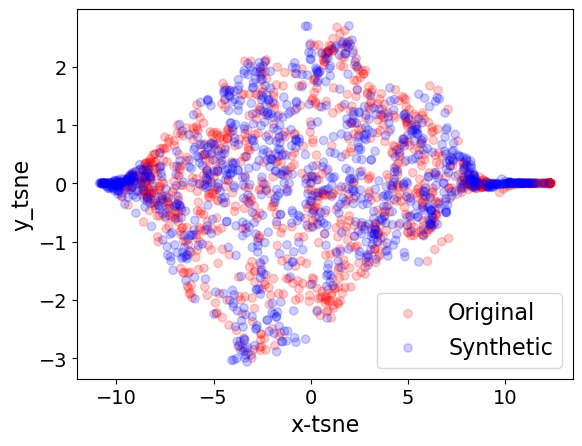}\label{fig:label_1_coscigan_tSNE}}\\
  \subfloat[{VFLGAN-TS (PCA)}]
  {\includegraphics[width=0.25\linewidth]{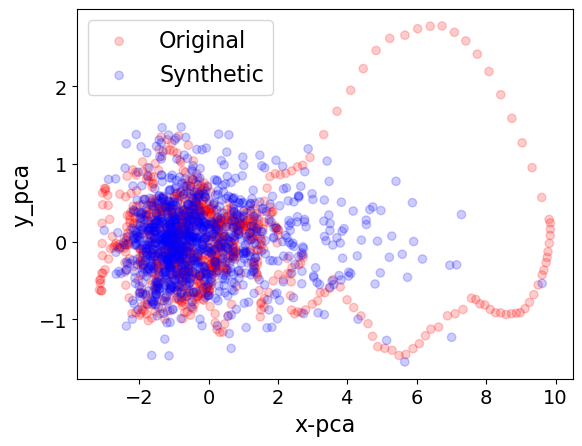}\label{fig:label_0_withSD_PCA}}     
  \subfloat[{VFLGAN-TS (t-SNE)}]
  {\includegraphics[width=0.25\linewidth]{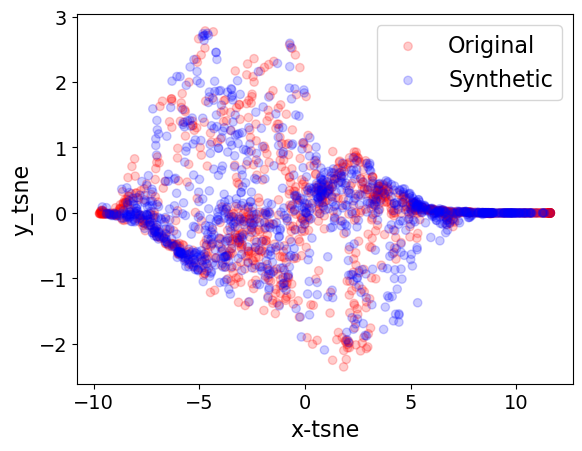}\label{fig:label_0_withSD_tSNE}}
  \subfloat[{VFLGAN-TS (PCA)}]
  {\includegraphics[width=0.25\linewidth]{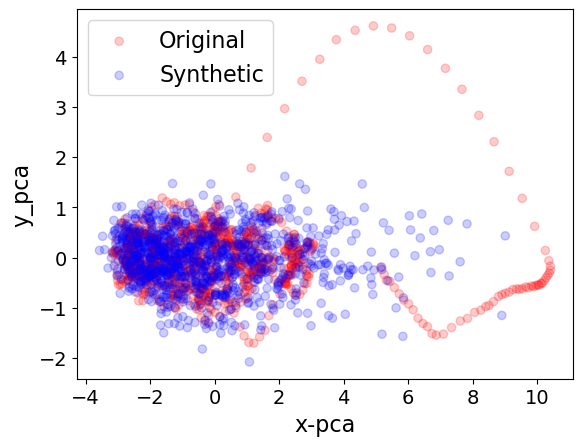}\label{fig:label_1_withSD_PCA}}     
  \subfloat[{VFLGAN-TS (t-SNE)}]
  {\includegraphics[width=0.25\linewidth]{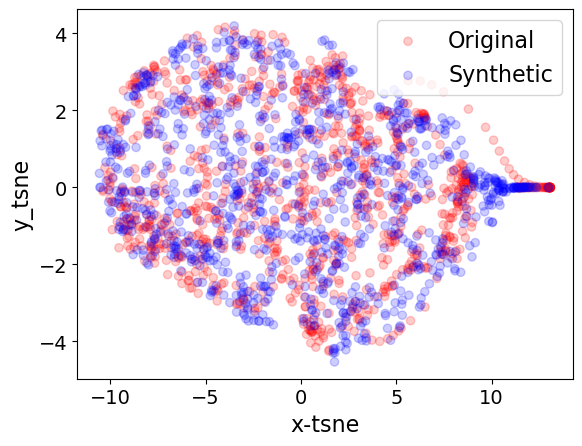}\label{fig:label_1_withSD_tSNE}}
  \caption{Visualization of similarity between real and synthetic datasets using PCA and t-SNE. The left two columns are results for EEG 0 and the right two columns are results for EEG 1.}
  \label{fig:visualization egg 0}
  \vspace{-6mm}
\end{figure*}

Samples in the EEG dataset do not have an explicit construction rule as Sine datasets. Thus, we can not construct the ground truth of the synthetic samples or compute MAE as \eqref{eq:MAE sine}. Instead, we use Performance on Downstream Tasks to evaluate the distribution similarity between synthetic and real datasets and the potential for real-world application. The downstream task for the EGG dataset is classification since the dataset was originally constructed for classification. First, we choose the parameters that achieve minimal AWD to generate synthetic datasets with the same number of samples as the real dataset. Then, we split all the datasets into training and test sets. Last, we train deep learning models on each training set and then test the models’ performance on test sets. We repeated the experiments ten times, and the average and standard deviation of the performances are presented in Table \ref{tab:classification_eeg}.
As shown in Table \ref{tab:classification_eeg}, VFLGAN-TS performs similarly to C-CosciGAN while significantly outperforming V-CosciGAN. The superior performance of VFLGAN-TS and C-CosciGAN comes from the fact that they can learn the feature correlation better since the three methods perform similarly in learning the distribution of each attribute at each time step, as shown in Fig. \ref{fig:wd_curve_eeg}. On the other hand, VFLGAN-TS outperforms VFLGAN since VFLGAN-TS can handle time series data more properly.

Lastly, we employ t-SNE and PCA to visualise the sample distribution, providing an intuitive depiction of the similarity between real and synthetic datasets. Figure \ref{fig:visualization egg 0} shows the distribution similarity between the real and synthetic datasets. According to the t-SNE visualization in Fig. \ref{fig:visualization egg 0}, although the performance of the three methods is similar, VFLGAN-TS is slightly better than the other two methods. According to the PCA visualization in Fig. \ref{fig:visualization egg 0}, none of the methods satisfactorily capture the distribution of outlier samples, but VFLGAN-TS performs similarly to C-CosciGAN, which is the upper limit for vertically partitioned methods.

\begin{figure}
    \centering
    \subfloat[Stock]
  {\includegraphics[width=0.5\linewidth]{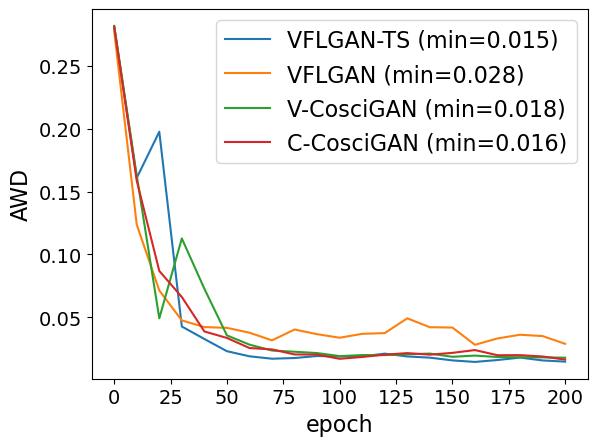}\label{fig:wd_curve_stock}}  
  \subfloat[Energy]
  {\includegraphics[width=0.5\linewidth]{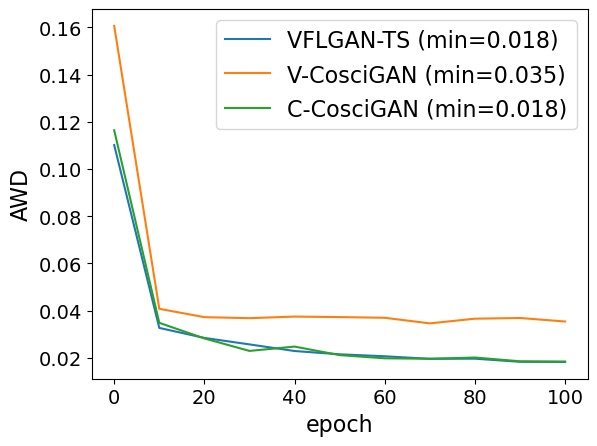}\label{fig:wd_curve_energy}}
    \caption{Wasserstein distance curves during training different methods on Stock and Energy Datasets. The AWD curve of VFLGAN is not shown in Fig. \ref{fig:wd_curve_energy} since the AWD increased to above 2.0 (epoch 100) from 0.33 (epoch 0).}
    \label{fig:wd_curve_stock&enegy}
    \vspace{-4mm}
\end{figure}

\begin{figure*} 
  \centering
  \subfloat[{V-CosciGAN (PCA)}]
  {\includegraphics[width=0.25\linewidth]{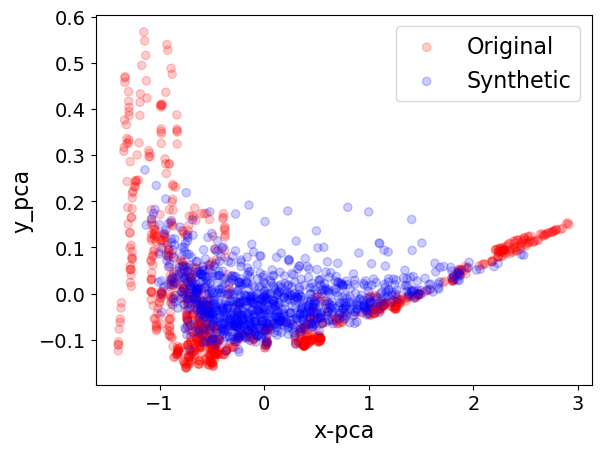}\label{fig:stock_withoutSD_PCA}}     
  \subfloat[{V-CosciGAN (t-SNE)}]
  {\includegraphics[width=0.25\linewidth]{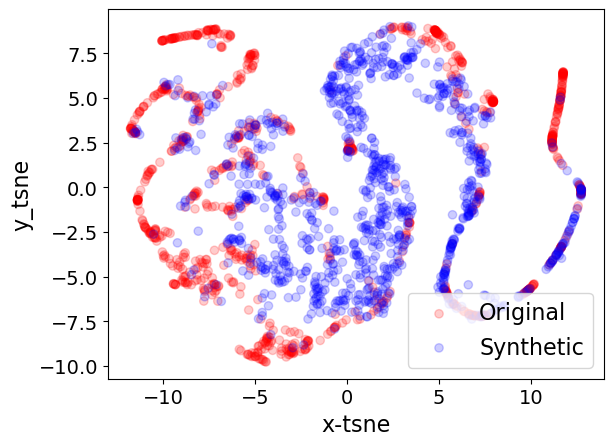}\label{fig:stock_withoutSD_tSNE}} 
  \subfloat[{V-CosciGAN (PCA)}]
  {\includegraphics[width=0.25\linewidth]{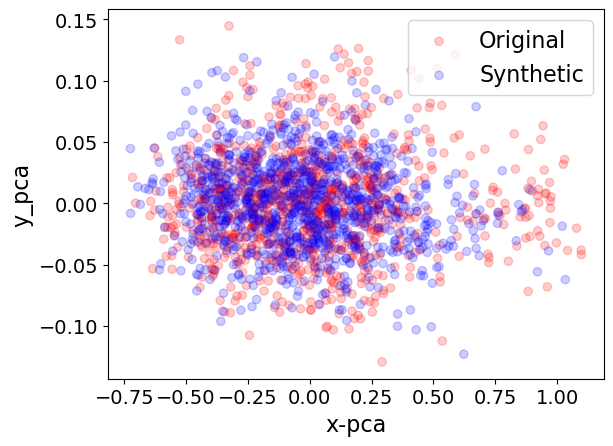}\label{fig:enegy_withoutSD_PCA}}     
  \subfloat[{V-CosciGAN (t-SNE)}]
  {\includegraphics[width=0.25\linewidth]{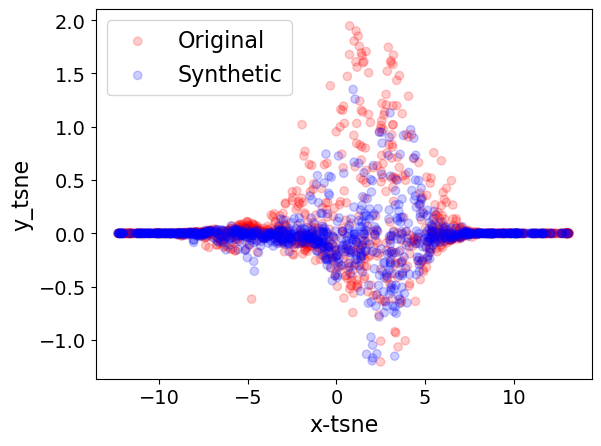}\label{fig:energy_withoutSD_tSNE}}\\
  \subfloat[{C-CosciGAN (PCA)}]
  {\includegraphics[width=0.25\linewidth]{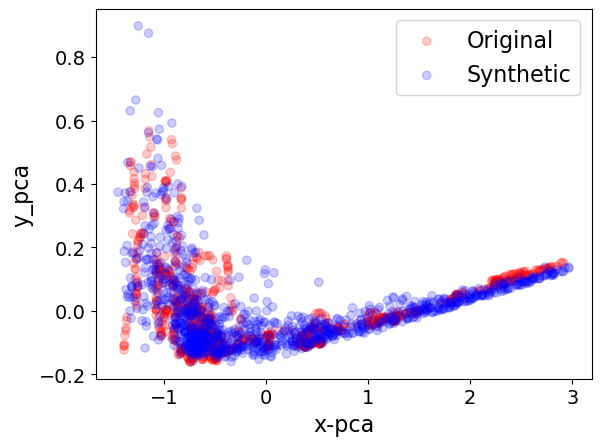}\label{fig:stock_coscigan_PCA}}     
  \subfloat[{C-CosciGAN (t-SNE)}]
  {\includegraphics[width=0.25\linewidth]{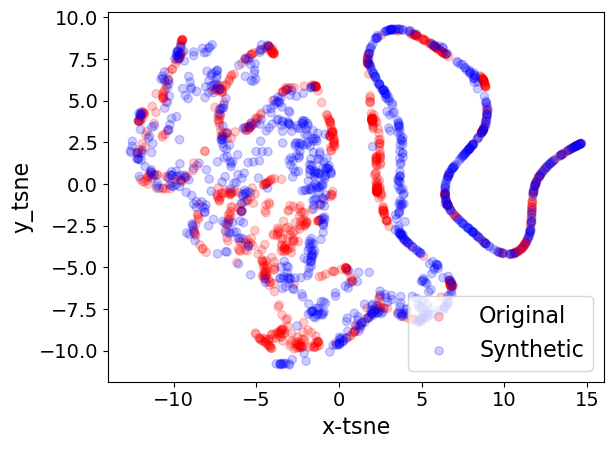}\label{fig:stock_coscigan_tSNE}} 
  \subfloat[{C-CosciGAN (PCA)}]
  {\includegraphics[width=0.25\linewidth]{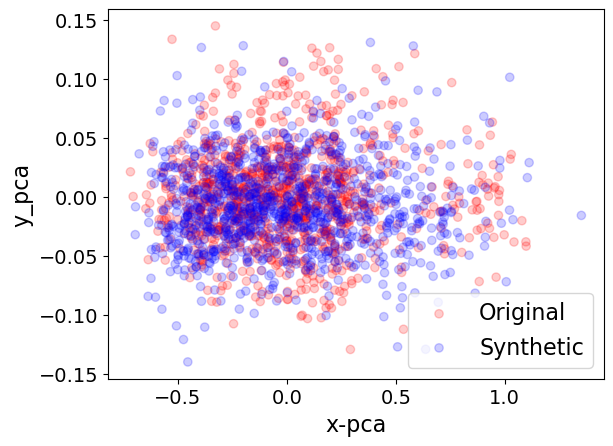}\label{fig:energy_coscigan_PCA}}     
  \subfloat[{C-CosciGAN (t-SNE)}]
  {\includegraphics[width=0.25\linewidth]{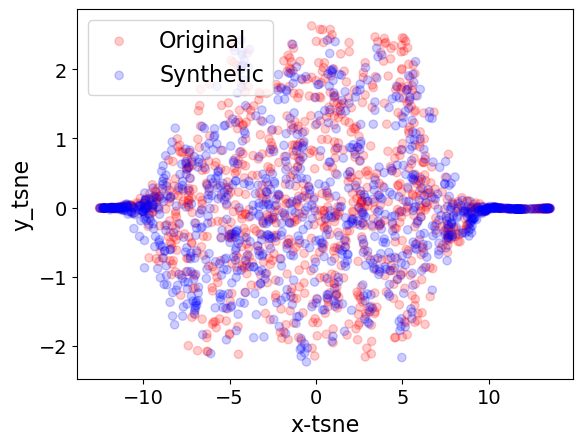}\label{fig:energy_coscigan_tSNE}}\\
  \subfloat[{VFLGAN-TS (PCA)}]
  {\includegraphics[width=0.25\linewidth]{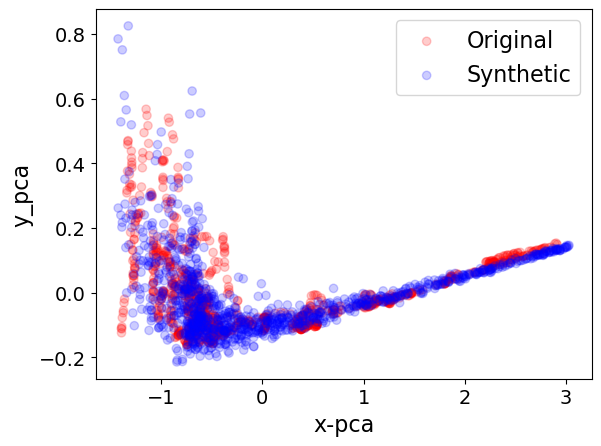}\label{fig:stock_withSD_PCA}}     
  \subfloat[{VFLGAN-TS (t-SNE)}]
  {\includegraphics[width=0.25\linewidth]{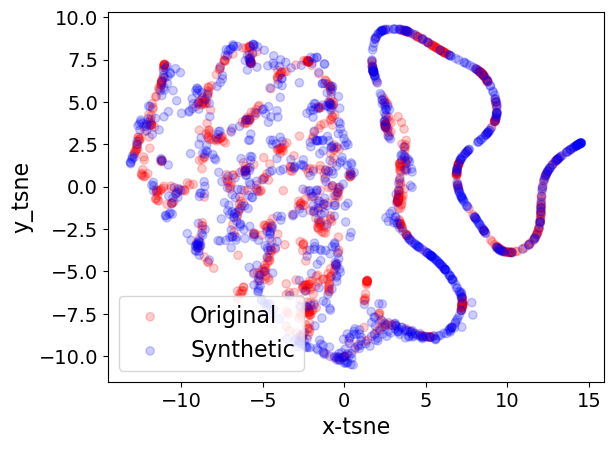}\label{fig:stock_withSD_tSNE}}
  \subfloat[{VFLGAN-TS (PCA)}]
  {\includegraphics[width=0.25\linewidth]{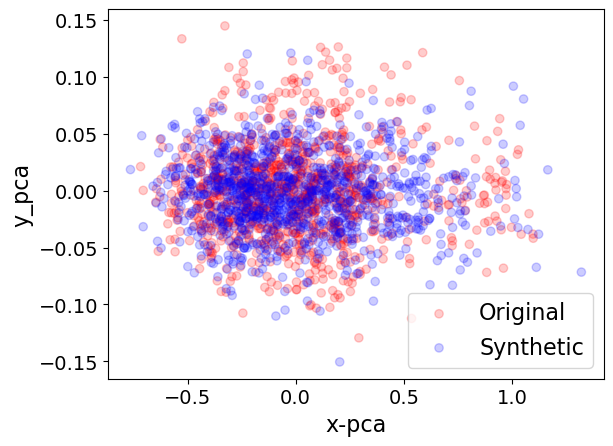}\label{fig:energy_withSD_PCA}}     
  \subfloat[{VFLGAN-TS (t-SNE)}]
  {\includegraphics[width=0.25\linewidth]{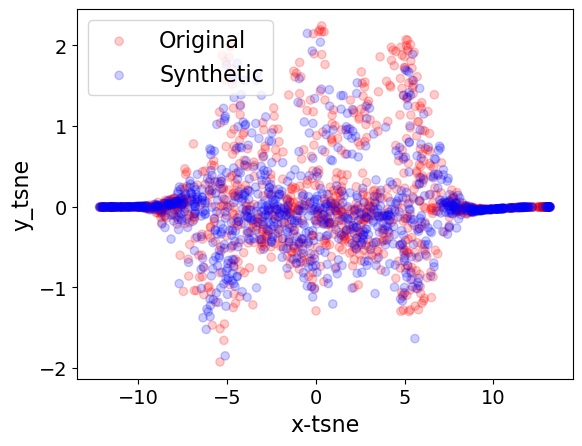}\label{fig:energy_withSD_tSNE}}
  \caption{Visualization of similarity between real and synthetic datasets using PCA and t-SNE. The left two columns are results for the Stock dataset and the right two columns are results for the Energy dataset.}
  \label{fig:visualization stock&Energy}
  \vspace{-4mm}
\end{figure*}

\subsubsection{Stock and Energy Datasets}
Stock and Energy datasets are real-world datasets, and the time-series length is 24. For these two datasets, we first present the AWD \eqref{eq:AWD} curves in Fig. \ref{fig:wd_curve_stock&enegy}. As shown in Fig. \ref{fig:wd_curve_stock}, the three CosciGAN-based methods perform similarly on the Stock dataset while outperforming VFLGAN. As shown in Fig. \ref{fig:wd_curve_stock}, VFLGAN-TS and C-CosciGAN outperform V-CosciGAN by a significant margin, and VFLGAN cannot converge on the Energy dataset. Based on Fig. \ref{fig:wd_curve_energy}, we can conclude that learning the attribute correlation, i.e., the shared discriminator in VFLGAN-TS and central discriminator in C-CosciGAN, can be beneficial to learning the single attribute distribution. The downstream task for Stock and Energy datasets is forecasting, i.e., forecasting the attribute value of the next time step based on the values of the past 23 time steps. First, we choose the parameters that achieve minimal AWD to generate synthetic datasets with the same number of samples as the real datasets. Then, we split all the datasets into training and test sets. Last, we train deep learning models on each training set and then test the models’ performance on test sets. We repeated the experiments ten times, and the average and standard deviation of the performances are presented in Table \ref{tab:forecasting_stock&energy}. VFLGAN-TS and C-CosciGAN outperform V-CosciGAN in both Stock and Energy datasets, as indicated in Table \ref{tab:forecasting_stock&energy}.
Furthermore, VFLGAN-TS performs superior to C-CosciGAN and VFLGAN on the Stock dataset. Lastly, we employ t-SNE and PCA to visualise the sample distribution in Fig. \ref{fig:visualization stock&Energy}, providing an intuitive depiction of the similarity between real and synthetic samples. According to the left two columns of Fig. \ref{fig:visualization stock&Energy}, the distribution of samples generated by VFLGAN-TS most closely matches that of the real samples, followed by C-CosciGAN and V-CosciGAN, in sequence. According to the right two columns of Fig. \ref{fig:visualization stock&Energy}, VFLGAN-TS and C-CosciGAN present similar performance and outperform V-CosciGAN.

\begin{table}
\centering
{
\renewcommand\arraystretch{1.}
    \caption{Mean Performances of ASSD \cite{Xun_vfl}}
    \label{tab:ASSD}
    \begin{tabular}{c|ccc|ccc}
    \hline
     {} & \textbf{Naive}  &\textbf{Corr} &\textbf{KNN} & \textbf{Naive}  &\textbf{Corr} &\textbf{KNN}\\
    \hline
    \textbf{Dataset}& \multicolumn{3}{c|}{\textbf{Record 1}} & \multicolumn{3}{c}{\textbf{Record 2}}\\
    \hline
    {EEG 0} & {0.53} & {0.53} &{0.56} & {0.53} & {0.54} &{0.50}\\
    
    \hline
    {EEG 1} & {0.50} & {0.45} &{0.53} & {0.51} & {0.47} &{0.54}\\

    \hline
    Stock & {0.46} & {0.51} &{0.55} & {0.49} & {0.49} &{0.56}\\

    \hline
    Energy & {0.47} & {0.48}  &{0.57} & {0.51} & {0.48} &{0.54}\\
    \hline
    \end{tabular}
    }
    \begin{tablenotes}
        \footnotesize
        \item Record 1 was selected by the method described in \cite{Theresa01} and Record 2 was selected by the method described in \cite{meeus2023achilles}. We repeat the experiment five times for naive and correlation features and ten times for KNN feature.
    \end{tablenotes}
    \vspace{-4mm}
\end{table}

\begin{table}
\centering
{
\renewcommand\arraystretch{1.}
    \caption{Mean and STD Performances of enhanced ASSD}
    \label{tab:ASSD enhanced}
    \begin{tabular}{c|c|ccc}
    \hline
     \textbf{Dataset} & \textbf{Method} & \textbf{Naive}  &\textbf{Corr} &\textbf{KNN} \\
    
    \hline
    \multirow{2}{*}{EEG 0} & {VFLGAN-TS} & {0.52(0.04)} & {0.53(0.04)} & {0.59(0.01)}  \\
    \cline{2-5}
    {} & {$(10, 10^{-3})$-DP} & {0.54(0.04)} & {0.49(0.02)} & {0.56(0.00)}\\
    \hline

    \multirow{2}{*}{EEG 1} & {VFLGAN-TS}& {0.48(0.05)} & {0.53(0.06)} & {0.57(0.01)}  \\
    \cline{2-5}
    {} & {$(10, 10^{-3})$-DP}  & {0.52(0.03)} & {0.48(0.05)} & {0.50(0.01)} \\
    \hline

    \multirow{2}{*}{Stock} & {VFLGAN-TS} &{0.52(0.01)} & {0.53(0.04)} & {0.60(0.01)}  \\
    \cline{2-5}
    {} & {$(10, 3\times 10^{-4})$-DP}   &{0.46(0.06)} & {0.48(0.03)} & {0.51(0.00)}\\
    \hline

    \multirow{2}{*}{Energy} & {VFLGAN-TS} & {0.82(0.03)} & {0.86(0.03)} & {0.59(0.00)}  \\
    \cline{2-5}
    {} & {$(10, 5\times 10^{-5})$-DP} & {0.49(0.03)} & {0.49(0.04)} & {0.52(0.00)} \\
    \hline
    \end{tabular}
    }
    \begin{tablenotes}
        \footnotesize
        \item Record 3 was selected by the proposed methods.
    \end{tablenotes}
    \vspace{-7mm}
\end{table}

\subsection{Privacy Analysis}

To evaluate the privacy breaches through synthetic datasets, we select two target samples, Record 1 and Record 2, for each dataset with the methods in \cite{Theresa01} and \cite{meeus2023achilles}, respectively. Then, we conduct ASSD with three features, including naive, correlation, and KNN, to attack synthetic datasets generated by VFLGAN-TS. In Table \ref{tab:ASSD}, a performance that is less than or equal to 0.5 means no privacy breaches and a higher performance indicates more privacy breaches. As shown in Table \ref{tab:ASSD}, the privacy beaches of the two target samples are limited. To enhance the privacy auditing, we apply method \eqref{eq:maximum minimum distance} to select the target sample for EEG 0 and Stock datasets and method \eqref{eq:kdd} to select target samples for EEG 1 and Energy datasets. Then, we conduct ASSD with the same three features. As shown in Table \ref{tab:ASSD enhanced}, the performance of naive and correlation features is poor for EEG 0, EEG 1, and Stock datasets, while the KNN feature performs better. On the other hand, compared to Table \ref{tab:ASSD}, Record 3 is more vulnerable to ASSD with KNN features than Record 1 and Record 2, which shows the effectiveness of our target sample selection methods. However, naive and correlation features outperform the KNN feature on Record 3 of Energy dataset. To reduce the privacy leakage, we use DPVFLGN-TS that follows $(10, \delta)$-DP, to generate synthetic datasets. As shown in Table \ref{tab:ASSD enhanced}, the privacy breaches decreased significantly. Here, we use the ASIF as proposed in \cite{Xun_vfl} to evaluate the privacy breach through the features transmitted to the server during the training process. As shown in Table \ref{tab:ASSD enhanced IF}, although the privacy breach of the target sample (Record 3 of Energy dataset) is significant through synthetic datasets, the privacy breach is minor through the features even for the non-DP method.

\begin{table}[htbp]
\centering
{
\renewcommand\arraystretch{1.}
    \caption{Mean Performances of ASIF through Features}
    \label{tab:ASSD enhanced IF}
    \begin{tabular}{c|c|cc|cc}
    \hline
     \multirow{2}{*}{Dataset} & \multirow{2}{*}{Method} & \multicolumn{2}{c|}{\textbf{Naive}}  &\multicolumn{2}{c}{\textbf{Corr}} \\
     
     {} & {} &\textbf{$F_1$} &\textbf{$F_2$} &\textbf{$F_1$} &\textbf{$F_2$} \\
    \hline
    \multirow{2}{*}{EEG 0} & {VFLGAN-TS} & 0.55 & 0.53 & 0.46 & 0.55 \\
    \cline{2-6}
    {} & {$(10, 10^{-3})$-DP} & 0.49 & 0.54 & 0.49 & 0.54\\
    \hline
    \multirow{2}{*}{EEG 1} & {VFLGAN-TS} & 0.48 &0.47 &0.53 & 0.46\\
    \cline{2-6}
    {} & {$(10, 10^{-3})$-DP} & 0.48 & 0.52 & 0.50 & 0.48\\
    \hline
    \multirow{2}{*}{Stock} & {VFLGAN-TS} &0.50 &0.51 & 0.52 & 0.51 \\
    \cline{2-6}
    {} & {$(10, 3\times 10^{-4})$-DP}  & 0.50 & 0.48 & 0.51 & 0.50 \\
    \hline
    \multirow{2}{*}{Energy} & {VFLGAN-TS} &0.49 & 0.47 & 0.46 & 0.49 \\
    \cline{2-6}
    {} & {$(10, 5\times 10^{-5})$-DP} & 0.54 & 0.54 & 0.50 & 0,52  \\
    \hline
    \end{tabular}
    }
    \begin{tablenotes}
        \footnotesize
        \item The target sample is Record 3 same as Table \ref{tab:ASSD enhanced}; $F_1$ and $F_2$ are the features from the feature extractor of Party 1 and Party 2, respectively.
    \end{tablenotes}
    \vspace{-4mm}
\end{table}

\section{Conclusion and Discussion}
This paper first proposed VFLGAN-TS, which is the \emph{first} generative model for publishing time-series data in vertically partitioned scenarios. Second, we provided a differentially private VFLGAN-TS to reduce privacy breaches. Third, we enhanced the privacy auditing scheme, which we previously proposed to audit VFLGAN, to evaluate the time-series synthetic data. The experimental results show that the performance of VFLGAN-TS is close to its upper limit, i.e., C-CosciGAN trained in a centralized manner, and DPVFLGAN-TS can significantly reduce privacy breaches. {One limitation of this paper is that neither centralised CosciGAN nor VFLGAN-TS can effectively capture the distribution of outlier samples, which is why the privacy leakage of some outliers is negligible. Future research could explore improving the model's ability to learn the distribution of inliers and outliers. Besides, although this paper made some effort to select the vulnerable samples, how to find the most vulnerable sample for privacy analysis is still an under-explored research topic.}


\bibliographystyle{IEEEtran}
\bibliography{reference}

\appendix

\subsection{Proof of Theorem 1}
This section will provide proof of Theorem 1, presented in the main paper. We need the following two propositions to prove Theorem 1, whose numbers follow the proposition number of the main paper.

\begin{myProp}
\label{Prop: Gaussian mechenism}
    (Gaussian Mechanism) Let $f: D \rightarrow R$ be an arbitrary function with sensitivity being     
    
    \centerline{$\Delta_2 f=\max _{D, D^{\prime}}\left\|f(D)-f\left(D^{\prime}\right)\right\|_2$}     
    \noindent for any adjacent $D, D^\prime \in \mathcal{D}$. The Gaussian Mechanism $M_\sigma$,     
    
    \centerline{$\mathcal{M}_\sigma(\boldsymbol{x})=f(\boldsymbol x)+\mathcal{N}\left(0, \sigma^2 I\right)$}     
    \noindent provides $\left(\alpha, \alpha \Delta_2 f^2 /2\sigma^2\right)$-RDP.
\end{myProp}

\begin{proof}
Proof can be found in \cite{mironov2017renyi}.
\end{proof}

\begin{myProp}
\label{theorem: post-processing}
(Post-processing). If $f(\cdot)$ satisfies ($\epsilon,\delta$)-DP, $g(f(\cdot))$ will satisfy ($\epsilon,\delta$)-DP for any function $g(\cdot)$. Similarly, if $f(\cdot)$ satisfies ($\alpha,\epsilon$)-RDP, $g(f(\cdot))$ will satisfy ($\alpha,\epsilon$)-RDP for any function $g(\cdot)$.
\end{myProp}

\begin{proof}
Proof can be found in \cite{dwork2014algorithmic}.
\end{proof}

Now, we repeat the Gaussian mechanism described in the main paper. When updating the parameters of the first linear layers of the local attribute discriminators and feature extractors, we clip and add Gaussian noise to the gradients as follows,
\begin{align}
\label{eq:clip}
    &clip(\mathcal{G}_{M}^1,C) = \mathcal{G}_{M}^1 / \max \left(1,\left\|\mathcal{G}_{M}^1\right\|_2 / C\right), \\
\label{eq:add noise}
    &\mathcal{G}_{M}^1 = clip(\mathcal{G}_{M}^1,C) + \mathcal{N}\left(0, \sigma^2 (2C)^2 I\right),
\end{align}
where the subscripte $M$ denote any local attribute discriminator and feature extractor, $\mathcal{G}_{M}^1$ denotes the gradients of the first layer parameters of the $M$, and $C$ denotes the clipping bound.

\begin{myTheo} \label{theorem RDP guarantee}
    (RDP Guarantee) All the local attribute discriminators and feature extractors satisfy ($\alpha$, $\alpha/(2\sigma^2)$)-RDP and the local attribute generators satisfy ($\alpha$, $\alpha/\sigma^2$)-RDP in one training iteration of DPVFLGAN-TS.
\end{myTheo}

\begin{proof}

Let $X_i \in \mathbb{R}^{N \times |A_i| \times T}$ denote the local data of party $i$ and $\boldsymbol{x}_i^B$ and $\boldsymbol{x}_i^{B^\prime} \in \mathbb{R}^{B \times |A_i| \times T}$ denote two adjacent mini-batches sampled from $X_i$, where $N$ is dataset size, $B$ is mini-batch size, $A_i$ is attribute size of party $i$, and $T$ is length of the time series. Then, $\boldsymbol{x_{ij}}^B \in \mathbb{R}^{B \times T}$ is the attribute that is input into $D_{ij}$ and $\boldsymbol{x_{i}}^B$ is the input of $FE_i$.

The gradients of the first-layer parameters of $M \in \{D_{i1},\cdots,D_{i|A_i|}, FE_i\}$ are clipped using \eqref{eq:clip}. Then, the L2 norm of those gradients has the following upper bound,
\begin{equation}
\label{eq:clip in proof}
    \left\|clip(\mathcal{G}{_{M(\boldsymbol{x}_j^B)}^1},C)\right\|_2 \leq C.
\end{equation}
According to the triangle inequality, the L2 sensitivity of the parameters can be derived as
\begin{equation}
\label{eq:sensitivity in proof}
    \Delta_2 f=\max _{\boldsymbol{x}_i^B, \boldsymbol{x}_i^{B^\prime}}\left\|clip(\mathcal{G}{_{M(\boldsymbol{x}_i^B)}^1},C)-clip(\mathcal{G}{_{M(\boldsymbol{x}_i^{B^\prime})}^1},C)\right\|_2 \leq 2 C.
\end{equation}
According to Proposition \ref{Prop: Gaussian mechenism}, $\mathcal{G}{_{M}^1}$ computed by \eqref{eq:add noise} satisfies $(\alpha, \alpha/(2\sigma^2))$-RDP. 

According to Proposition \ref{theorem: post-processing}, the parameters of the first layer of $M$ updated by,
\begin{equation}
    \boldsymbol{\theta}_{M}^1 = \boldsymbol{\theta}_{M}^1-\eta_{M}\mathcal{G}_{M}^1,
\end{equation}
satisfy the same RDP as $\mathcal{G}_{M}^1$. 
The function of $M$, $f_M$, can be expressed as,
\begin{equation}
\label{eq:InF-DP}
    {f}_M=func_{M}(\boldsymbol{\theta}_{M}^1 \boldsymbol{x}_{i(j)}), 
\end{equation}
where $\boldsymbol{x}_{i(j)}$ denotes the input of $M$, $func_{M}$ denotes the calculation after the first layer, i.e., $\boldsymbol{\theta}_{M}^1 \boldsymbol{x}_i(j)$. Thus, according to Proposition \ref{theorem: post-processing}, since $\boldsymbol{\theta}^1_{D_i}$ satisfies $(\alpha, \alpha/(2\sigma^2))$-RDP, the mechanism $f_M$ in \eqref{eq:InF-DP} satisfy $(\alpha, \alpha/(2\sigma^2))$-RDP, i.e., all the local attribute discriminators and feature extractors satisfy $(\alpha, \alpha/(2\sigma^2))$-RDP.

On the other hand, the local attribute generator, $G_{ij}$, is trained by the corresponding discriminator $D_{ij}$ and feature extractor $FE_i$. During the back-propagation process, let $\boldsymbol{\delta}^{1}_{FE_i}$ and $\boldsymbol{\delta}^1_{D_i}$ denote the backward gradients after the first layer of $FE_i$ and $D_i$, respectively. Then, the backward gradients for $G_{ij}$, $\boldsymbol{\delta}^{G_{ij}}$, can be calculated by,
\begin{equation}
    \boldsymbol{\delta}^{G_{ij}} = \boldsymbol{\delta}^1_{D_i}\boldsymbol{\theta}_{D_i}^1 + \boldsymbol{\delta}^1_{FE_i}\boldsymbol{\theta}_{FE_i}^1.
\end{equation}
According to Proposition 1 (in the main paper), $\boldsymbol{\delta}^{G_{ij}}$ satisfies $(\alpha, \alpha/\sigma^2)$-RDP.
Since the parameters of $G_{ij}$ is updated according to $\boldsymbol{\delta}^{G_{ij}}$, ${G_{ij}}$ satisfies $(\alpha, \alpha/\sigma^2)$-RDP.
\end{proof}

\subsection{Threat Model of the Auditing Scheme}

We consider the scenario where all models, including attribute generators, attribute discriminators, feature extractors, and the shared discriminator, are kept private while the generated synthetic dataset is publicly accessible; that is, the attacker can only access the synthetic dataset. 

In \cite{Theresa01}, the authors proposed a shadow model-based membership inference attack that assumes the adversary has an auxiliary dataset with a similar distribution to the private datasets and is much larger than the private dataset. However, this assumption is impractical for data owners since they intend to use most data to train the model rather than audit privacy breaches. Thus, this paper uses the Leave-One-Out (LOO) assumption proposed in \cite{Jiayuan01}, which is much stronger than the auxiliary data assumption. In LOO assumption, the attacker knows the whole training dataset but \textbf{one} target sample and aims to guess the existence of the target sample.

\textbf{Takeaway:} LOO is a strong assumption. The privacy breaches of the synthetic dataset will be negligible under realistic assumptions if the attack success rate is low under the LOO assumption.

\subsection{Construction of Six-Feature Sine Dataset}

\textbf{Synthetic Sine Datasets}: Similar to the two-attribute Sine dataset, the six-attribute dataset is constructed as follows,
\begin{align}
    \label{eq:sine data 2}
    &F_{1i} = A sin(2\pi f_{1i} \bm t)+ \epsilon, \\
    \label{eq:sine data 3}
    &F_{2i} = A sin(2\pi f_{2i} \bm t)+ \epsilon, \quad  i \in \{1,2,3\},
\end{align}
where $F_{ji}$ denotes the attribute $i$ stored in Party $j$, $A$ can be sampled from $\mathcal{N}(0.4,0.05)$ or $\mathcal{N}(0.6,0.05)$ with equal propability, $\epsilon\in\mathcal{N}(0,0.05)$, and $[f_{11}, f_{12}, f_{13}, f_{21}, f_{22}, f_{23}]=[0.01, 0.005, 0.0075, 0.0125, 0.015, 0.0175]$. In summary, in the two-party scenario, each party stores three attributes, and the amplitudes of the six attributes are the same for a given sample. We can utilize this characteristic to evaluate whether the generative methods learn the correlation between the six attributes stored in different parties. We generate 1,024 samples for both class, i.e., $A\in\mathcal{N}(0.4,0.05)$ and $A\in\mathcal{N}(0.6,0.05)$, and each attribute consists of 800 time steps ($\bm t = [0,1, \cdots, 799]$ in \eqref{eq:sine data 2} and \eqref{eq:sine data 3}).

\end{document}